\documentclass[twocolumn,notitlepage,floatfix,reprint,superscriptaddress]{revtex4-2}

\usepackage[lmargin=.7in,rmargin=.7in,tmargin=.6in,bmargin=0.8in]{geometry}
\usepackage{graphicx}
\usepackage{amsmath}
\usepackage{amsthm}
\usepackage{amssymb}
\usepackage{latexsym}
\usepackage{array}
\usepackage{hyperref}
\usepackage{amsfonts}
\usepackage{dsfont}
\usepackage{mathrsfs}
\usepackage{verbatim}
\usepackage{bbold}
\usepackage[normalem]{ulem}
\usepackage{upgreek}
\usepackage{makecell}
\usepackage{adjustbox}
\usepackage{algorithm}
\usepackage{algpseudocode}
\usepackage{color}
\usepackage{algcompatible}

\usepackage{textcomp}
\usepackage[dvipsnames]{xcolor}

\usepackage{bm}
\usepackage{times}

\newcommand{\bra}[1]{\left<#1\right|}
\newcommand{\ket}[1]{\left|#1\right>}
\newcommand{\abs}[1]{\left|#1\right|}
\newcommand{\norm}[1]{\left\lVert#1\right\rVert}

\newcommand{\braket}[2]{\left<{#1}|{#2}\right>}
\newcommand{\ketbra}[2]{\ket{#1}\!\!\bra{#2}}

\newcommand{\tr}[1]{\operatorname{Tr}\!\left[#1\right]}
\newcommand{\idx}{\mathrm{idx}}

\newtheorem{theorem}{Theorem}

\newtheorem{lemma}{Lemma}
\newtheorem{corollary}{Corollary}
\newtheorem{definition}{Definition}

\makeatletter
\renewcommand\part[1]{%
  \clearpage
  \onecolumngrid
  \section*{#1}
  
}
\makeatother

\makeatletter
\let\origaddcontentsline\addcontentsline
\renewcommand{\addcontentsline}[3]{}
\makeatother

\begin{document}

\part{}

\title{Physics‑Aware Learnability: From Set‑Theoretic Independence to Operational Constraints}

\author{Jeongho~Bang}\email{jbang@yonsei.ac.kr}
\affiliation{Institute for Convergence Research and Education in Advanced Technology, Yonsei University, Seoul 03722, Republic of Korea}
\affiliation{Department of Quantum Information, Yonsei University, Incheon 21983, Republic of Korea}

\author{Kyoungho~Cho}
\affiliation{Institute for Convergence Research and Education in Advanced Technology, Yonsei University, Seoul 03722, Republic of Korea}
\affiliation{Department of Statistics and Data Science, Yonsei University, Seoul 03722, Republic of Korea}

\date{\today}

\begin{abstract}
Beyond binary classification, learnability can become a logically fragile notion: in EMX, even the class of all finite subsets of $[0,1]$ is learnable in some models of ZFC and not in others. We argue the paradox is operational. The standard definitions quantify over arbitrary set-theoretic learners that implicitly assume non-operational resources (infinite precision, unphysical data access, and non-representable outputs). We introduce \emph{physics-aware learnability} (PL), which defines the learnability relative to an explicit access model---a family of admissible physical protocols. Finite-precision coarse-graining reduces continuum EMX to a countable problem, via an exact pushforward/pullback reduction that preserves the EMX objective, making the independence example provably learnable with explicit $(\epsilon,\delta)$ sample complexity. For quantum data, admissible learners are exactly POVMs on $d$ copies, turning sample size into copy complexity and yielding Helstrom(-type) lower bounds. For finite no-signaling and quantum models, PL feasibility becomes linear or semidefinite and is therefore decidable.
\end{abstract}

\maketitle

Identifying learnability is a foundational promise of learning theory: given a learning task, can finite data guarantee optimal performance with high probability? In binary classification, the PAC/VC framework makes this promise unusually crisp by reducing learnability to a finitary combinatorial witness with sharp sample-complexity bounds~\cite{Valiant1984,Vapnik1971,Blumer1989,Shalev2014}. Beyond classification, however, learnability can become logically fragile: in EMX (estimating the maximum), Ben-David \emph{et al.} showed that even the class of all finite subsets of $[0,1]$ is learnable in some models of ZFC set theory but not in others~\cite{BenDavid2019}. In other words, a finite-sample question can hinge on axioms external to the learning problem, and EMX learnability admits no general VC-like characterization of finite character.

We argue that this pathology is operational rather than statistical. The standard definitions quantify over arbitrary set-theoretic learners---functions from samples to hypotheses---on uncountable domains; in the infinite-precision limit, this can smuggle a ``ghost in the machine'', i.e., non-operational resources such as unphysical data access and outputs that are not finitely nameable. However, real learners are physical devices: they interact with the environment through a finite-precision interface and are constrained, for example, by no-cloning and measurement back-action in quantum settings~\cite{WoottersZurek1982,Dieks1982}, and/or by causal structure in distributed settings~\cite{PopescuRohrlich1994,Barrett2005}. We therefore introduce \emph{physics-aware learnability} (PL), which keeps the success criterion fixed but makes data access explicit via a family of admissible protocols. This shift yields (i) an exact coarse-graining reduction that collapses continuum EMX to a countable problem and turns the ZFC-independence example into a provably learnable finite-precision task with explicit $(\epsilon,\delta)$ sample complexity, (ii) a quantum characterization in which admissible learners are precisely POVMs on $d$ copies, leading to Helstrom copy-complexity lower bounds, and (iii) decidability of PL feasibility in finite operational models via convex optimization.

We note that because the main text is intentionally compact, see the Supplementary Information (SI) for full proofs and more detailed descriptions. A main text--to--SI correspondence map is provided in Methods.

\section*{Results}

\subsection*{Physics-aware learnability (PL): a minimal operational interface}

PL separates what counts as success from what actions are allowed. A learning problem is specified by environments, hypotheses and a utility function. Herein, a physical access model specifies which input-output behaviors are realizable at a given resource budget.

\begin{definition}[Learning task and optimal value]
A learning task is a triple $(\Theta,\mathcal{H},U)$ where $\Theta$ is a set of environments, $\mathcal{H}$ is a set of hypotheses, and $U:\Theta\times\mathcal{H} \to [0,1]$ is a utility. For $\theta \in \Theta$, define $\mathrm{opt}_{\mathcal{H}}(\theta) := \sup_{h\in\mathcal{H}}U(\theta,h)$.
\label{def:task}
\end{definition}

The components have distinct operational meanings: $\theta\in\Theta$ is the unknown ``state of the world'' (a distribution, a quantum state/channel, a no-signaling box, etc); $h \in \mathcal{H}$ is the learner's final output, which must be nameable by a finite classical record, so we treat $\mathcal{H}$ as \emph{representable} (at most countable, with an implicit encoding); and $U(\theta,h)\in[0,1]$ is the performance score, with $\mathrm{opt}_{\mathcal{H}}(\theta)$ as the benchmark achievable if $\theta$ were known. Crucially, {\bf Definition~\ref{def:task}} specifies \emph{what} near-optimal means, not \emph{how} it is achieved---the latter is determined by the access model.

\begin{definition}[Admissible protocol family]
Fix a task $(\Theta,\mathcal{H},U)$. An admissible protocol family is a sequence $\mathfrak{L}=\{\mathfrak{L}_d\}_{d \in \mathbb{N}}$, where each $\mathfrak{L}_d$ is a set of Markov kernels $Q(\cdot\mid\theta) \in \Delta(\mathcal{H})$ mapping environments to distributions over hypotheses. We assume $\mathfrak{L}_d$ is convex and closed under classical post-processing.
\end{definition}

In PL, a protocol is identified with its observable behavior: the conditional output law $Q(\cdot|\theta)$ over hypotheses. All randomness, adaptivity, and internal dynamics are absorbed into $Q$, while physical restrictions are encoded by the set $\mathfrak{L}_d$ of kernels achievable with resource budget $d$. Convexity and closure under classical post-processing express that one may randomize between protocols and freely reprocess any classical output; different physical theories correspond to different choices of $\mathfrak{L}$ (classical i.i.d., $d$-copy quantum~\cite{Ciliberto2018QML,Zhao2024Learning}, no-signaling~\cite{BangChoJae2026}, coarse-grained access~\cite{Gray2002Quantization}, etc). With these primitives in place, we can now state the learnability notion that PL proposes.

\begin{definition}[Physics-aware learnability (PL)]
For $\epsilon,\delta \in (0,1)$, the task $(\Theta,\mathcal{H},U)$ is ``$(\epsilon,\delta)$-learnable in PL'' relative to $\mathfrak{L}$ if there exist a resource budget $d$ and a kernel $Q \in \mathfrak{L}_d$, such that, for all $\theta\in\Theta$,
\begin{eqnarray}
\Pr_{H\sim Q(\cdot\mid\theta)}\Big[U(\theta,H) \ge \mathrm{opt}_{\mathcal{H}}(\theta)-\epsilon\Big] \ge 1-\delta.
\label{eq:PL_def}
\end{eqnarray}
\end{definition}

When $\mathfrak{L}$ is chosen to coincide with classical i.i.d.\ sampling (the standard EMX interface), Eq.~(\ref{eq:PL_def}) reproduces ordinary EMX learnability (SI {\bf Theorem~6}).


\begin{figure*}[t]
\centering
\includegraphics[width=0.60\linewidth]{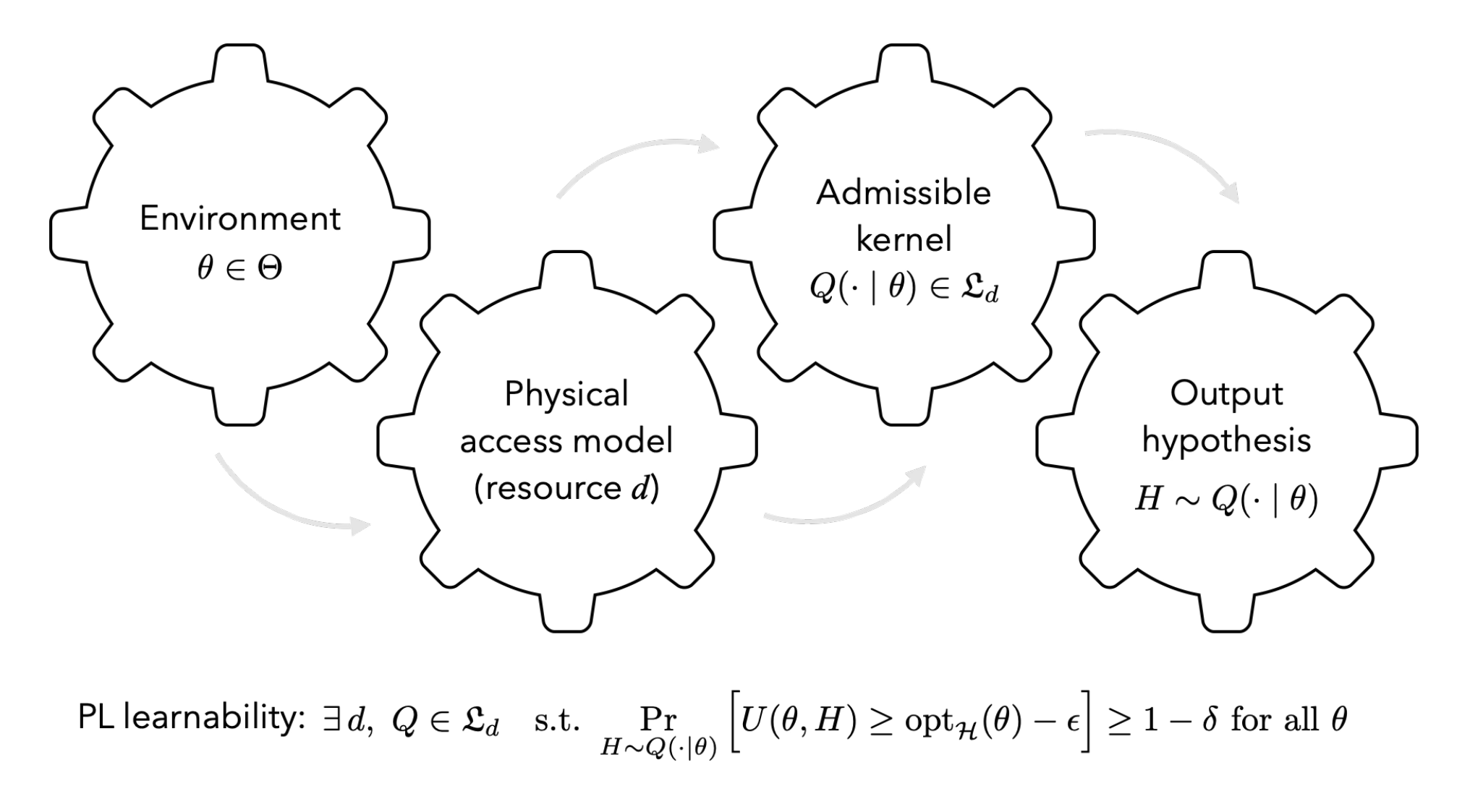}
\caption{{\bf Schematic of physics-aware learnability (PL).} A learning task is specified by $(\Theta,\mathcal{H},U)$: environments $\theta\in\Theta$, a representable (finitely nameable) hypothesis set $\mathcal{H}$, and a utility $U(\theta,h)\in[0,1]$. A physical access model fixes, for each resource budget $d$, a set $\mathfrak{L}_d$ of admissible input--output behaviors, represented as Markov kernels $Q(\cdot|\theta)\in\Delta(\mathcal{H})$. Any concrete protocol---possibly randomized, adaptive, quantum, or distributed---induces exactly one such kernel; convexity and closure under classical post-processing capture free randomization and classical relabeling of outcomes. PL asks whether there exist $d$ and $Q\in\mathfrak{L}_d$ such that, uniformly for all $\theta$, the output $H\sim Q(\cdot|\theta)$ achieves near-optimal performance with high probability, $\Pr[U(\theta,H)\ge \mathrm{opt}_{\mathcal{H}}(\theta)-\epsilon]\ge 1-\delta$, where $\mathrm{opt}_{\mathcal{H}}(\theta)=\sup_{h\in\mathcal{H}}U(\theta,h)$. This relocates the existential quantifier from arbitrary set-theoretic maps on samples to operational behaviors at the laboratory boundary: the objective $U$ (what counts as success) is held fixed while the admissible physics enters only through $\mathfrak{L}$. Different choices of $\mathfrak{L}$ recover classical i.i.d. sampling, finite-precision coarse-graining interfaces, $d$-copy quantum access (POVM-induced kernels), or finite no-signaling models; accordingly, $d$ becomes a genuine resource complexity (samples/copies/queries) that can both eliminate unphysical independence phenomena and expose genuine information-theoretic limits.}
\label{fig:schematic_PL}
\end{figure*}

In PL, the learnability is witnessed by an admissible \emph{behavior}---a kernel $Q\in\mathfrak{L}_d$---rather than an arbitrary set-theoretic map on samples, so all randomness and adaptivity are packaged into an operational input--output law. As a result, learnability is explicitly relative to an access model: it is a property of the pair $((\Theta,\mathcal{H},U),\mathfrak{L})$, with all physical assumptions confined to $\mathfrak{L}$ while the utility $U$ continues to encode what is valued. This is precisely why the Ben-David independence can disappear under realistic constraints: when $\mathfrak{L}_d$ is defined by finitary physics (e.g., POVMs on $d$ copies or functions of coarse-grained observations), PL no longer quantifies over abstract functions on an uncountable domain and the independence class becomes operationally learnable. The budget $d$ then becomes a genuine resource complexity (sample, copy, or query complexity) that measures the cost of admissible learning (See Fig.~\ref{fig:schematic_PL}).

\subsection*{Finite precision collapses undecidable EMX into a provable operational task}

Ben-David's paradox is triggered by treating continuum observations as exact and allowing learners to be arbitrary functions on an uncountable domain.

For clarity, recall the EMX objective in the present notation. Let $X$ be a domain and let $\mathcal{F}\subseteq\{0,1\}^X$ be a proper hypothesis family (we view each $f\in\mathcal{F}$ as a subset of $X$). For an unknown finitely supported distribution $P$ on $X$, define $P(f):=\mathbb{E}_{x\sim P}[f(x)]$ and $\mathrm{opt}_{\mathcal{F}}(P):=\sup_{f\in\mathcal{F}}P(f)$. An $(\epsilon,\delta)$-EMX learner observes $d$ i.i.d.\ samples $x_1,\ldots,x_d\sim P$ and outputs $f\in\mathcal{F}$ such that $P(f)\ge \mathrm{opt}_{\mathcal{F}}(P)-\epsilon$ with probability at least $1-\delta$. This is exactly the PL task with utility $U(P,f)=P(f)$; here we keep the objective fixed and modify only the access via $\pi$.

Operationally, access is mediated by a finite-precision interface, modeled as a coarse-graining map (measurement interface) $\pi:X \to Y$ with \emph{countable} $Y$ (bins, pixels, finite-resolution readouts), so the learner receives $\pi(x)$ rather than $x$. To connect this restriction to standard learnability results without weakening the EMX objective, we need an exact {\em reduction} between the continuum task on $X$ (under $\pi$-access) and the induced discrete task on $Y$. The following theorem provides this pushforward/pullback equivalence, preserving both the achieved utility and the optimum.

\begin{theorem}[Coarse-graining reduction for EMX]
\label{thm:coarse}
Let $\pi:X \to Y$ with countable $Y$. For a distribution $P$ on $X$, let $Q=\pi_{\#}P$ be the pushforward on $Y$. (Notationally, this $Q$ is a distribution on $Y$; PL output kernels appear as $Q(\cdot\mid\theta)$.) For any family $\mathcal{G} \subseteq \{0,1\}^{Y}$, define its pullback $\pi^{-1}(\mathcal{G})=\{g \circ \pi:g \in \mathcal{G}\} \subseteq \{0,1\}^{X}$. If $\mathcal{G}$ is $(\epsilon,\delta)$-EMX learnable over $\pi_{\#}\Theta := \{\pi_{\#}P: P \in \Theta\}$ by a standard learner on $Y$, then $\pi^{-1}(\mathcal{G})$ is $(\epsilon,\delta)$-learnable in PL over $\Theta$ under the access model that reveals only $\pi(x)$.
\end{theorem}

\begin{proof}[Proof sketch]---Run the EMX learner on the discretized sample $(\pi(x_1),\ldots,\pi(x_d))$ to output $g \in \mathcal{G}$ and return $f=g \circ \pi$. The identity $P(g \circ \pi)=Q(g)$ implies that both the achieved utility and the optimum are preserved by pushforward/pullback. Full proof: SI {\bf Theorem~8}.
\end{proof}

This exact reduction is the pivot: once $\pi$ is fixed, the continuum task collapses to an EMX instance on the countable alphabet $Y$, and ZFC-provable learnability on $Y$ transfers verbatim to an admissible PL learner on $X$. Consider the finite-subset class on $Y$, $\mathcal{F}^{Y}_{\mathrm{fin}}=\{F \subseteq Y: \text{$F$ finite}\}$. Since $Y$ is countable, weak EMX learnability is provable in ZFC and can be witnessed by explicit monotone compression schemes (SI Sec.~III and SI {\bf Corollary~1}). Pulling back to $X$ yields the finite-precision analogue of the Ben-David class,
\begin{eqnarray}
\mathcal{F}^{(\pi)}_{\mathrm{fin}} := \pi^{-1}(\mathcal{F}^{Y}_{\mathrm{fin}}) = \{\pi^{-1}(F):F\subseteq Y\ \text{finite}\}.
\label{eq:finite_precision_class}
\end{eqnarray}

\begin{corollary}[Finite-precision EMX is provably learnable]
\label{cor:finite_precision}
For any coarse-graining $\pi:X \to Y$ with countable $Y$, the class $\mathcal{F}^{(\pi)}_{\mathrm{fin}}$ is $(1/3,1/3)$-learnable in PL under the access model that reveals only $\pi(x)$.
\end{corollary}

\begin{proof}[Proof sketch]---Apply {\bf Theorem~\ref{thm:coarse}} with $\mathcal{G}=\mathcal{F}^{Y}_{\mathrm{fin}}$ and use weak EMX learnability on countable domains (SI {\bf Corollary~1}). Full proof: SI Sec.~IV.F.
\end{proof}

A quantitative message also survives the translation. Once a naming scheme (an injection $\idx:Y \to \mathbb{N}$) is part of the interface, a one-line ``quantile'' rule achieves explicit sample complexity: on a discretized sample $y_1,\ldots,y_d$, output the initial segment $\{y:\idx(y) \le \max_j \idx(y_j)\}$. For finitely supported distributions, this captures at least $1-\epsilon$ probability with probability at least $1-\delta$ using
\begin{eqnarray}
d \ \ge\ \frac{\ln(1/\delta)}{-\ln(1-\epsilon)} = O\left(\frac{1}{\epsilon}\ln\frac{1}{\delta}\right) \quad (\text{small $\epsilon$}),
\label{eq:quantile_sample}
\end{eqnarray}
as proved in SI {\bf Theorem~9}. Operationally, the undecidability disappears because the physically meaningful problem no longer quantifies over functions that depend on non-finitely-accessible distinctions. Instead, it is the discrete EMX problem induced by $\pi$.

\subsection*{Quantum access: admissible learners are POVMs and samples become copy complexity}

In quantum learning, ``data'' are physical systems~\cite{Huang2021PowerData}. If an environment prepares an unknown state $\hat{\rho}_{\theta}$ on a finite-dimensional Hilbert space $\mathcal{K}$, then a $d$-sample budget corresponds to having access to $\hat{\rho}_{\theta}^{\otimes d}$. Crucially, the no-cloning theorem forbids generating additional i.i.d. copies from fewer specimens~\cite{WoottersZurek1982,Dieks1982}. In PL terms, the admissibility is not ``any function of the sample,'' but it is ``any quantum protocol acting on $\hat{\rho}_{\theta}^{\otimes d}$ and outputting a classical label''~\cite{Ciliberto2018QML}.

This admissible set has an exact characterization.

\begin{theorem}[Quantum PL kernels are POVM-induced]
\label{thm:povm}
Fix a finite-dimensional Hilbert space $\mathcal{K}$ and a finite or countable hypothesis set $\mathcal{H}$. A kernel $Q(h\mid\hat{\rho})$ is realizable by a quantum protocol acting on $\hat{\rho}^{\otimes d}$ and outputting $h \in \mathcal{H}$ if and only if there exists a POVM $\{\hat{M}_h\}_{h \in \mathcal{H}}$ on $\mathcal{K}^{\otimes d}$, such that
\begin{eqnarray}
Q(h\mid\hat{\rho})=\tr{\hat{M}_h \hat{\rho}^{\otimes d}} \quad  (\forall h).
\label{eq:born_rule_kernel}
\end{eqnarray}
\end{theorem}

\begin{proof}[Proof sketch]---Any quantum-to-classical protocol can be represented as a CPTP map to a classical register, and the adjoint of this map sends basis projectors to POVM elements; conversely, measuring a POVM implements the kernel via the Born rule. Full proof: SI {\bf Theorem~7}.
\end{proof}

{\bf Theorem~\ref{thm:povm}} highlights a methodological contrast with set-theoretic learners. In PL, the admissibility is explicit and grounded in physics. Moreover, the sample-complexity parameter $d$ becomes \emph{copy complexity}: a hard physical resource due to no-cloning~\cite{Bang2026RunLength}.

This shift also exposes intrinsically physical limits. Consider binary state identification, where $\theta \in \{0,1\}$ indexes one of two states $\hat{\rho}_0, \hat{\rho}_1$ and the hypothesis is a guess $h \in \{0,1\}$. Under $d$-copy access, any protocol is a two-outcome POVM $\{\hat{M}_0, \hat{M}_1\}$ on $\mathcal{K}^{\otimes d}$. A worst-case Helstrom inequality bounds the simultaneous success probability for the two environments~\cite{Helstrom1976,Watrous2018}:
\begin{eqnarray}
\tr{\hat{M}_0 \hat{\rho}_0^{\otimes d}} + \tr{\hat{M}_1 \hat{\rho}_1^{\otimes d}} \le 1 + \frac{1}{2}\norm{\hat{\rho}_0^{\otimes d} - \hat{\rho}_1^{\otimes d}}_1.
\label{eq:helstrom}
\end{eqnarray}
For non-orthogonal pure states with overlap $\gamma \in (0,1)$, one has $\norm{\hat{\rho}_0^{\otimes d} - \hat{\rho}_1^{\otimes d}}_1 = 2\sqrt{1-\gamma^{2d}}$ (SI {\bf Lemma~3}), implying that to achieve small error probability $\delta$ uniformly over $\theta$, one needs $d=\Omega(\log(1/\delta)/\abs{\log\gamma})$ copies (SI {\bf Corollary~4}). In other words, PL clarifies a second kind of impossibility that has nothing to do with set-theoretic independence: some tasks are limited by the geometry of state space and cannot be made reliable with finitely many copies~\cite{Huang2020Predicting,Zhao2024Learning}.

\subsection*{Operational constraints can restore decidability in finite models}

The EMX paradox is a logical independence statement about unconstrained learners on infinite domains. PL suggests a different and operationally relevant meaning of deciding learnability: given a finite description of the environment set, hypothesis set and admissible physics, can we algorithmically decide whether a protocol exists that meets a target guarantee?

In many physically motivated scenarios, the admissible behaviors form a convex set described by finitely many constraints. For example, no-signaling correlations in finite input-output alphabets are defined by linear equalities and inequalities~\cite{PopescuRohrlich1994,Barrett2005,Brunner2014Bell}. The quantum admissibility with fixed $d$ is semidefinite-representable via POVM constraints ({\bf Theorem~\ref{thm:povm}}). In such settings, PL feasibility becomes explicit convex feasibility.

\begin{theorem}[Decidability of PL feasibility in finite operational models]
\label{thm:decide}
Let $\Theta$ and $\mathcal{H}$ be finite, and fix $\epsilon,\delta \in (0,1)$ and a budget $d$. Suppose $\mathfrak{L}_d$ is specified as either (i) a rational polytope in the kernel variables $Q(h\mid\theta)$ (including normalization and non-negativity), or (ii) a quantum $d$-copy model in which $Q(h\mid\theta) = \tr{\hat{M}_h \hat{\rho}_\theta^{\otimes d}}$ for some POVM $\{\hat{M}_h\}$. Then, deciding whether there exists $Q \in \mathfrak{L}_d$ satisfying Eq.~(\ref{eq:PL_def}) reduces to (i) linear feasibility, or (ii) semidefinite feasibility, respectively.
\end{theorem}

\begin{proof}[Proof sketch]---For each $\theta$, define the $\epsilon$-optimal set $\mathcal{G}_\theta(\epsilon):=\{h\in\mathcal{H}:U(\theta,h)\ge \mathrm{opt}_{\mathcal{H}}(\theta)-\epsilon\}$. Then Eq.~(\ref{eq:PL_def}) is equivalent to the linear constraints $\sum_{h\in\mathcal{G}_\theta(\epsilon)} Q(h\mid\theta)\ge 1-\delta$ for all $\theta$ (since $\Theta$ and $\mathcal{H}$ are finite). Under (i), $\mathfrak{L}_d$ is also described by linear equalities/inequalities, so the decision problem is linear feasibility (LP). Under (ii), the admissibility is captured by linear matrix inequalities for $\{\hat{M}_h\}$, while the PL constraints remain linear in $\{\hat{M}_h\}$ through the trace pairing, yielding semidefinite feasibility (SDP). Full proof: SI {\bf Theorems~11--12}.
\end{proof}

The message is not that physics ``solves'' set theory. Rather, once the admissible interface is specified by finite constraints, the relevant question becomes the device feasibility: does a point in a finite-dimensional convex set satisfy a finite set of performance inequalities? In this regime, learnability becomes decidable (and often efficiently computable) by standard convex optimization~\cite{Boyd2004Convex}.

\section*{Discussion}

Ben-David \emph{et al.} showed that, in EMX, learnability of a simple class can be independent of the ZFC axioms~\cite{BenDavid2019}. PL reframes this result as a warning label about a particular modelling choice: if we insist on treating a learner as an arbitrary set-theoretic function on an uncountable domain, then the statement ``there exists a learner'' can hinge on axioms external to the learning problem. On the other hand, the operational question that an experimenter asks is different: given a concrete interface to the world, does there exist a protocol consistent with that interface that achieves the target guarantee?

PL clarifies this distinction by moving the existential quantifier to the laboratory boundary. Under finite precision, the continuum EMX instance collapses to a discrete problem induced by the coarse-graining map $\pi$, and the corresponding operational task becomes provably learnable with explicit sample complexity ({\bf Theorem~\ref{thm:coarse}}, {\bf Corollary~\ref{cor:finite_precision}}, and Eq.~(\ref{eq:quantile_sample})). This does not refute the set-theoretic independence but localizes it. Independence attaches to the non-operational limit in which the interface is the identity on an uncountable set and the learner is allowed to depend on distinctions no finite record can resolve.

At the same time, PL also exposes the impossibility statements that are genuinely physical rather than logical. In quantum settings, no-cloning turns the sample budget into an irreducible resource, and the admissible learners reduce exactly to POVMs on $d$ copies ({\bf Theorem~\ref{thm:povm}})~\cite{Ciliberto2018QML}. The basic identification tasks then exhibit the copy-complexity lower bounds enforced by state geometry (Eq.~(\ref{eq:helstrom})). These bounds are not artefacts of set theory; they reflect fundamental information constraints.

Finally, PL changes what it means to decide learnability. Once the admissibility is described by finitely many operational constraints, existence questions become concrete feasibility problems over convex sets. For finite no-signaling models this is linear programming; for finite-dimensional quantum models it is semidefinite programming ({\bf Theorem~\ref{thm:decide}}). This suggests a pragmatic outlook: for many physically specified learning scenarios, the right question is not whether learnability is provable in a fixed axiom system, but whether a device meeting a specification exists within an explicit admissible set.

Beyond the EMX paradox, PL points to a broader program for machine intelligence. Data access is part of the learning problem. Coarse-graining is not merely a nuisance but a design variable that trades resolution against learnability and resource cost. Quantum access is not a minor implementation detail. Rather, it reshapes admissible inference and introduces new resources, such as copy complexity. More generally, physical constraints can carve out the operational regimes in which different physics-dependent ``dimension-like'' invariants may exist, not as absolute characterizations of learnability across all mathematical models, but as robust descriptors within constrained access models~\cite{BangChoJae2026}.

\begin{quote}
``\emph{Physics decides what is learnable}''
\end{quote}
In PL, this slogan is not rhetorical: it is a reminder that ``learnability'' is always a statement about a task \emph{and} an interface, and that changing the interface can both eliminate unphysical paradoxes and expose genuine information constraints.

\section*{Methods}

\subsection*{Where to find full proofs (Supplementary Information map)}

The complete proofs and extended background are provided in the Supplementary Information (SI), which is the extended manuscript accompanying this main text. For transparency, we indicate where the complete arguments appear in the SI.
\begin{itemize}
\item Coarse-graining reduction ({\bf Theorem~\ref{thm:coarse}}): SI Sec.~IV.E and SI {\bf Theorem~8}.
\item Finite-precision EMX learnability and explicit quantile learner ({\bf Corollary~\ref{cor:finite_precision}}, Eq.~(\ref{eq:quantile_sample})): SI Secs.~III.B, IV.F, V.A and SI {\bf Theorem~9} / {\bf Corollaries~2--3}.
\item POVM characterization of quantum admissibility ({\bf Theorem~\ref{thm:povm}}): SI Sec.~IV.D and SI {\bf Theorem~7}.
\item Copy-complexity lower bounds for state identification (Eq.~(\ref{eq:helstrom})): SI Sec.~V.B, SI {\bf Theorem~10} and SI {\bf Corollary~4}.
\item Decidability in finite polytope/SDP models ({\bf Theorem~\ref{thm:decide}}): SI Sec.~V.C and SI {\bf Theorems~11--12}.
\end{itemize}

\subsection*{Conventions}

Following Ref.~\cite{BenDavid2019}, the EMX statements are formulated for finitely supported distributions, avoiding measurability issues over arbitrary subsets.

\section*{Acknowledgement}

This work was supported by the Ministry of Science, ICT and Future Planning (MSIP) by the National Research Foundation of Korea (RS-2024-00432214, RS-2025-03532992, and RS-2025-18362970) and the Institute of Information and Communications Technology Planning and Evaluation grant funded by the Korean government (RS-2019-II190003, ``Research and Development of Core Technologies for Programming, Running, Implementing and Validating of Fault-Tolerant Quantum Computing System''), the Korean ARPA-H Project through the Korea Health Industry Development Institute (KHIDI), funded by the Ministry of Health \& Welfare, Republic of Korea (RS-2025-25456722). We acknowledge the Yonsei University Quantum Computing Project Group for providing support and access to the Quantum System One (Eagle Processor), which is operated at Yonsei University.


%


\clearpage
\newpage
\onecolumngrid

\makeatletter
\let\addcontentsline\origaddcontentsline
\makeatother

\part{Supplementary Information}

This Supplementary Information is intended to serve as a companion to our main paper, ``Physics Aware Learnability from Set Theoretic Independence to Operational Constraints''. It is written so that readers can follow the motivation, definitions, and technical steps without having to fill in missing details. We provide intermediate calculations, extended proofs, and additional explanations that support the compact presentation in the main text.

\tableofcontents

\section{Introduction}

A central aspiration of learning theory is to delineate, with mathematical precision, which inference tasks are possible and/or which are not from finite data. This aspiration has both conceptual and practical force: on the conceptual side, it seeks a principled boundary between generalization and mere curve fitting; on the practical side, it guides the design of algorithms by isolating structural properties of hypothesis classes and data-generating processes that control statistical complexity. Nowhere is this program more successful than in binary classification, where the PAC/VC framework yields a strikingly sharp characterization of learnability and sample complexity in terms of a single combinatorial invariant~\cite{Valiant1984,Vapnik1971,Blumer1989,Shalev2014}. The existence of such a characterization has shaped the modern view that ``learnability'' is a mathematically well-posed property of a learning problem, and that one can hope for similarly robust characterizations in richer learning settings.

However, the robustness of learnability as a mathematical notion becomes subtle once one departs from the classical classification paradigm. Ben-David \emph{et al.} demonstrated that, for a natural generalization of PAC learning known as the ``estimating the maximum (EMX)'' problem, there exist remarkably simple hypothesis families for which the statement ``this class is learnable'' is independent of the standard axioms of mathematics (that is, of Zermelo--Fraenkel set theory with the axiom of choice, or ZFC set theory)~\cite{BenDavid2019}. In particular, in one model of ZFC, the class of all finite subsets of $[0,1]$ is learnable in the EMX sense, while in another model of ZFC, it is not. This is not merely a technical pathology tied to exotic constructions: the proof leverages the independence of statements about the cardinality of the continuum, and it implies that there can be no ``dimension-like'' quantity of finite character whose finiteness captures EMX learnability in full generality. The unsettling aspect of this phenomenon is that it undermines an implicit methodological assumption that once a learning problem has been formalized in a standard mathematical framework, the question of learnability is a determinate mathematical fact.

A closer look reveals that the source of the independence is not ``learning'' per se, but rather the particular formalization of the learning that abstracts away implementation. The standard definitions typically represent a learner as an arbitrary ``function'' from finite samples to hypotheses, i.e.,
\begin{eqnarray}
G:\bigcup_{m\in\mathbb{N}}\mathcal{Z}^m \to \mathcal{H},
\label{eq:learner_functional_form}
\end{eqnarray}
where $\mathcal{Z}$ denotes the sample (data) space, $\mathcal{H}$ the hypothesis class, and $G$ the learner that maps any finite sample sequence to a hypothesis. All such functions are quantified when defining the learnability. This choice is deliberate: by treating the learner as an unconstrained function, one separates statistical questions from computational considerations. This separation is instrumental in classical results such as the VC characterization. Yet, as emphasized by Ben-David \emph{et al.} in Ref.~\cite{BenDavid2019}, the existence questions for functions over infinite domains are logically delicate, and can depend on set-theoretic axioms in a way that has no operational counterpart. From the standpoint of learning as an empirical science, this suggests that a purely set-theoretic notion of the learnability may be too permissive: it allows ``learners'' that have no interpretation as realizable procedures, and it attributes to them a form of existence whose truth can hinge on axioms that are not fixed by the learning problem itself.

This observation motivates a shift in perspective. Every learner is ultimately a physical system that acquires, stores, and processes information. Even in classical settings, this immediately imposes constraints of \emph{effective realizability}: the learner must be representable by a finite description and executable by a feasible physical process. In quantum settings, the need to incorporate physical constraints is even more pronounced, because the nature of data and the admissible operations are dictated by fundamental principles. For example, when training data are quantum states, the no-cloning principle rules out the routine classical idealization that one may freely duplicate an unknown specimen~\cite{WoottersZurek1982,Dieks1982}. More generally, quantum measurement disturbs the system and constrains what information can be extracted from finitely many copies~\cite{NielsenChuang2010,Watrous2018,Bang2026RunLength} (Additionally, see Refs.~\cite{Biamonte2017QML,Ciliberto2018QML,Huang2021PowerData} for quantum-machine-learning overviews that emphasize data access and measurement as central constraints. See also Ref.~\cite{Sweke2021} for a rigorous comparison of quantum and classical efficient learnability for discrete distributions in the PAC framework). In distributed and relativistic scenarios, no-signaling constraints limit the flow of information and the admissible forms of correlation that a learning protocol can exploit~\cite{PopescuRohrlich1994,Barrett2005,BangChoJae2026}. These are not peripheral implementation details but structural constraints on what ``access to data'' means in the first place~\cite{Barrett2005,Watrous2018}.

The aim of this work is to articulate a notion of \emph{physics-aware learnability} that treats learnability as an operational property relative to a specified set of admissible physical processes. Conceptually, this amounts to replacing the unconstrained quantification in Eq.~(\ref{eq:learner_functional_form}) by a restricted class $\mathfrak{L}_{\rm phys}$ of learners that are implementable within a given physical theory and resource regime,
\begin{eqnarray}
G \in \mathfrak{L}_{\rm phys},
\label{eq:physical_learner_constraint}
\end{eqnarray}
and then asking whether there exists such a learner that achieves the desired generalization guarantee. The intent is not to ``resolve'' set-theoretic independence by fiat, nor to assert that physical principles decide cardinal arithmetic. Rather, the goal is to align the formal semantics of learnability with the operational meaning of learning: if learnability is meant to capture what can be achieved from finite observations by an admissible procedure, then admissibility---computational and physical---must be part of the definition.

This physics-aware viewpoint offers a principled way to reinterpret independence phenomena. When a learnability question is undecidable in the ZFC sense, one should ask whether the undecidability persists after restricting attention to physically realizable learners, or whether it is an artifact of allowing non-constructive, non-operational objects. At the same time, importing physical constraints does not trivialize the theory: it can introduce new, genuinely operational barriers (for instance, limitations induced by no-cloning or no-signaling) and can shift the frontier from set-theoretic existence to computability, complexity, and information-theoretic feasibility. In this sense, physics-aware learnability is not a minor refinement of existing theory, but a foundational stance: it treats the mathematical study of learning as inseparable from the physical laws that govern information.

\section{Background: Learnability Beyond PAC, and the Ben-David Undecidability Phenomenon}

This section recalls the classical PAC/VC paradigm and then introduces the EMX (estimating the maximum) problem as a canonical ``beyond-PAC'' learning task in which the familiar methodology of characterizing learnability by a finite combinatorial invariant can fail in a dramatic, set-theoretic sense. The technical core of the Ben-David \emph{et al.} phenomenon is an equivalence between (weak) EMX learnability and the existence of a particular compression primitive, which in turn can be tied to cardinality statements about the underlying domain. Here we therefore present, in a self-contained manner, the EMX formulation, the notion of monotone compression, and the central learnability--compression equivalence with full proofs.

\subsection{PAC learnability and the VC paradigm}

We begin by recalling the classical PAC formulation, both to fix notation and to highlight the feature that will fail in the EMX setting: in PAC learning, ``learnability'' can be characterized by a \,\emph{finite}\, combinatorial witness.

Let $\mathcal{X}$ be an instance space and let $\mathcal{H}\subseteq\{0,1\}^{\mathcal{X}}$ be a hypothesis class. Data are drawn i.i.d. from an unknown distribution $\mathcal{D}$ on $\mathcal{X}\times\{0,1\}$. For a hypothesis $h\in\mathcal{H}$, its (population) classification error is
\begin{eqnarray}
\operatorname{err}_{\mathcal{D}}(h) := \Pr_{(x,y)\sim\mathcal{D}}\big[h(x)\neq y\big].
\label{eq:pac_population_error}
\end{eqnarray}
In the realizable case, one assumes there exists a target $h^{\star}\in\mathcal{H}$ with $\operatorname{err}_{\mathcal{D}}(h^{\star})=0$. A learning rule takes a labeled sample $S=((x_1,y_1),\ldots,(x_m,y_m)) \sim \mathcal{D}^m$ and outputs $\widehat{h}=G(S) \in \mathcal{H}$. The PAC guarantee asks that, for all $\epsilon,\delta \in (0,1)$, there exists a sample size $m=m(\epsilon,\delta)$ such that
\begin{eqnarray}
\Pr_{S\sim\mathcal{D}^m}\big[\operatorname{err}_{\mathcal{D}}(G(S))\le \epsilon\big] \ge 1-\delta \quad \text{for all realizable}~\mathcal{D}.
\label{eq:pac_realizable_guarantee}
\end{eqnarray}

The central structural invariant in this theory is the Vapnik--Chervonenkis (VC) dimension. A finite set $\{x_1,\ldots,x_k\} \subseteq \mathcal{X}$ is shattered by $\mathcal{H}$ if for every labeling vector $(b_1,\ldots,b_k) \in \{0,1\}^k$ there exists $h \in \mathcal{H}$ such that $h(x_i)=b_i$ for all $i$. The VC dimension $\mathrm{VC}(\mathcal{H})$ is the largest $k$ that can be shattered (or $\infty$ if no such finite maximum exists)~\cite{Vapnik1971,Blumer1989,Shalev2014}. The fundamental theorem of PAC learning states that $\mathcal{H}$ is PAC learnable if and only if $\mathrm{VC}(\mathcal{H})<\infty$; moreover, finite VC dimension yields explicit sample complexity bounds for uniform convergence and empirical risk minimization~\cite{Valiant1984,Vapnik1971,Blumer1989,Shalev2014}.

Two features of this paradigm are worth emphasizing. First, both the definition of VC dimension and the mechanism that turns it into generalization bounds rely on finitary witnesses (finite shattered sets, finite growth bounds). Second, the learner is treated as an arbitrary map from samples to hypotheses (ignoring computational issues), yet the learnability criterion remains mathematically robust because the needed existence statements can be certified via finitary combinatorics.

A natural program is to extend PAC-style reasoning beyond binary classification: keep the meta-structure of ``choose a hypothesis that performs well under an unknown distribution from i.i.d. samples,'' but replace the classification error by other performance objectives. One might expect that an appropriate VC-like finite invariant continues to characterize learnability in such general settings. The EMX framework of Ben-David \emph{et al.} shows that this expectation can fail in a stark way~\cite{BenDavid2019}: even for very simple hypothesis families, the statement ``this class is learnable'' can become sensitive to set-theoretic axioms.

\subsection{The EMX problem}

Let $\mathcal{X}$ be a domain set and let $\mathcal{F} \subseteq \{0,1\}^{\mathcal{X}}$ be a family of $\{0,1\}$-valued functions. We identify each $f\in\mathcal{F}$ with the subset $\{x \in \mathcal{X}:f(x)=1\}$ and conversely. Let $P$ be a probability distribution over $\mathcal{X}$. For $f\in\mathcal{F}$, we write
\begin{eqnarray}
\mathbb{E}_{x\sim P}\big[f(x)\big] = P(f) := P\big(\{x:f(x)=1\}\big).
\label{eq:expectation_as_measure}
\end{eqnarray}
The EMX objective is to output, from i.i.d.\ samples, a function in $\mathcal{F}$ whose expectation under $P$ is close to the best achievable in $\mathcal{F}$. Formally, define
\begin{eqnarray}
\textrm{opt}_{\mathcal{F}}(P) := \sup_{f \in \mathcal{F}} P(f).
\label{eq:opt_emx}
\end{eqnarray}

A simple motivating example of EMX is the ``ads and website visitors'' problem. Following the motivating discussion in Ref.~\cite{BenDavid2019}, imagine a website whose potential visitors form the domain $\mathcal{X}$. Each advertisement $A$ in a fixed pool targets a population $F_A \subseteq \mathcal{X}$ (sports fans, travelers, etc.), and the site owner observes a sample of visitors drawn i.i.d. from an unknown distribution $P$ over $\mathcal{X}$. The objective is to pick an ad whose target population has near-maximal visit probability. Identifying each $A$ with its target set $F_A$, the pool becomes a family $\mathcal{F}=\{F_A\}$, and choosing the ``best'' ad becomes exactly the EMX task: output some $F \in \mathcal{F}$ with $P(F)$ close to $\sup_{F'\in\mathcal{F}}P(F')$ based only on the sample.

A direct measure-theoretic formulation would require measurability of each $f$ with respect to $P$. Following Ref.~\cite{BenDavid2019}, and because the undecidability phenomenon we discuss is set-theoretic rather than analytic, we adopt the standard simplification that all distributions under consideration are \emph{finitely supported} and are defined over the $\sigma$-algebra of all subsets of $\mathcal{X}$. Under this convention, the measurability is automatic and $P(f)$ is well-defined for every $f\subseteq \mathcal{X}$.

A (proper) learner is represented as a function that maps a finite sample sequence to an element of $\mathcal{F}$,
\begin{eqnarray}
G:\bigcup_{n\in\mathbb{N}}\mathcal{X}^n \to \mathcal{F}.
\label{eq:emx_learner_map}
\end{eqnarray}
Properness is crucial in EMX: if outputs outside $\mathcal{F}$ were allowed, the all-ones function would trivially achieve $\textrm{opt}_{\mathcal{F}}(P)=1$ whenever it is admissible, collapsing the problem.

\begin{definition}[EMX learnability]
Fix $\epsilon,\delta\in(0,1)$. We say that $G$ is an $(\epsilon,\delta)$-EMX learner for $\mathcal{F}$ if there exists an integer $d=d(\epsilon, \delta)$ such that for every finitely supported distribution $P$ over $\mathcal{X}$,
\begin{eqnarray}
\Pr_{S\sim P^d}\left[ P\big(G(S)\big) \ge \textrm{opt}_{\mathcal{F}}(P)-\epsilon \right] \ge 1-\delta.
\label{eq:emx_definition}
\end{eqnarray}
We say that $\mathcal{F}$ is EMX learnable if such a learner exists for all $\epsilon,\delta$.
\end{definition}

In the independence phenomenon of Ref.~\cite{BenDavid2019}, the emphasis is on a constant-parameter ``weak'' learnability notion, such as $(\epsilon, \delta)=(1/3,1/3)$, because the undecidability already arises at this coarse level.

\subsection{Monotone compression schemes}

The compression schemes are a classical bridge between learnability and finitary structure~\cite{LittlestoneWarmuth1986}. The version central to~\cite{BenDavid2019} is a monotone compression scheme tailored to families of finite sets. It is helpful to think of monotone compression as a two-stage game protocol. A ``compressor'' (Alice) sees $m$ sample points $x_1,\ldots,x_m$ that all lie in some unknown target set $F\in\mathcal{F}$ and is allowed to keep only $d$ of them. A ``reconstructor'' (Bob) then maps the retained $d$ points to a hypothesis $\eta(\cdot)\in\mathcal{F}$ that must contain \emph{all} original sample points.
The key requirement is \emph{monotonicity}: reconstruction outputs a \emph{superset} of the sample, not just a set that is consistent with it. This is exactly the form of compression that connects cleanly to EMX, where the objective is to cover as much probability mass as possible.

Here, we assume that $\mathcal{F}$ is a family of finite subsets of $\mathcal{X}$, and we consider samples as sequences $(x_1,\ldots,x_m) \in \mathcal{X}^m$ with the understanding that the underlying set of observed points is $\{x_1,\ldots,x_m\}$.

\begin{definition}[Monotone compression]
Let $d\le m$ be integers. An $m\to d$ monotone compression scheme for $\mathcal{F}$ is a reconstruction function
\begin{eqnarray}
\eta:\mathcal{X}^d \to \mathcal{F}
\label{eq:reconstruction_function}
\end{eqnarray}
such that for every $F\in\mathcal{F}$ and every sequence $x_1,\ldots,x_m\in F$, there exist indices $1\le i_1<\cdots<i_d\le m$ for which
\begin{eqnarray}
\{x_1,\ldots,x_m\} \;\subseteq\; \eta(x_{i_1},\ldots,x_{i_d}).
\label{eq:monotone_compression_condition}
\end{eqnarray}
The adjective ``monotone'' reflects that the reconstructed set is required to contain (rather than merely be consistent with) the full sample.
\end{definition}

A mild closure condition on $\mathcal{F}$ is needed to connect compression back to EMX learning.

\begin{definition}[Union boundedness]
A family $\mathcal{F}$ is \emph{union bounded} if for every $F_1,F_2\in\mathcal{F}$ there exists $F_3 \in \mathcal{F}$ such that $F_1\cup F_2\subseteq F_3$. In particular, by induction, for every finite collection $F_1,\ldots,F_r\in\mathcal{F}$ there exists $\widetilde{F}\in\mathcal{F}$ satisfying $\bigcup_{j=1}^r F_j \subseteq \widetilde{F}$.
\end{definition}

\subsection{Learnability $\Longleftrightarrow$ monotone compression}

We now state and prove, in full, the equivalence that forms the technical entry point to the Ben-David \emph{et al.} phenomenon. The equivalence is formulated for the weakest nontrivial parameters $(1/3,1/3)$, which suffice for undecidability.

\begin{lemma}[Weak EMX learnability $\Longleftrightarrow$ weak monotone compression]
\label{lem:learnability_compression_equivalence}
Let $\mathcal{F}$ be a union bounded family of finite subsets of $\mathcal{X}$. The following are equivalent:
\begin{itemize}
\item[(i)] $\mathcal{F}$ is $(1/3,1/3)$-EMX learnable.
\item[(ii)] There exists an $(m+1) \to m$ monotone compression scheme for $\mathcal{F}$ for some $m \in \mathbb{N}$.
\end{itemize}
\end{lemma}

\begin{proof}[Proof of (ii)$\Rightarrow$(i)]
Assume that $\mathcal{F}$ admits an $(m+1) \to m$ monotone compression scheme with reconstruction function $\eta_{m+1}:\mathcal{X}^m \to \mathcal{F}$.

\medskip
\paragraph*{\bf Step 1 (Boosting: $(m+1)\to m$ implies $n\to m$ for all $n\ge m+1$).}
We show by induction on $n$ that there exists a reconstruction function $\eta_n:\mathcal{X}^m \to \mathcal{F}$ such that for every $F\in\mathcal{F}$ and every sequence $x_1,\ldots,x_n \in F$ there are indices $i_1 < \cdots < i_m$ with
\begin{eqnarray}
\{x_1,\ldots,x_n\}\subseteq \eta_n(x_{i_1},\ldots,x_{i_m}).
\label{eq:n_to_m_compression}
\end{eqnarray}
The base case $n=m+1$ is exactly the assumed scheme $\eta_{m+1}$. For the inductive step, assume $\eta_n$ exists for some $n \ge m+1$ and let $x_1,\ldots,x_{n+1} \in F$. By the inductive hypothesis applied to the first $n$ points, there exist indices $i_1 < \cdots < i_m \le n$ such that
\begin{eqnarray}
\{x_1,\ldots,x_n\} \subseteq \eta_n(x_{i_1},\ldots,x_{i_m}).
\label{eq:inductive_cover_n}
\end{eqnarray}
Define the intermediate $m+1$-tuple
\begin{eqnarray}
U := \big(x_{i_1},\ldots,x_{i_m},x_{n+1}\big),
\label{eq:U_tuple}
\end{eqnarray}
whose entries all lie in $F$. Applying the $(m+1)\to m$ scheme to $U$ yields indices $j_1<\cdots<j_m$ (each $j_t\in\{1,\ldots,m+1\}$) such that
\begin{eqnarray}
\{x_{i_1},\ldots,x_{i_m},x_{n+1}\} \subseteq \eta_{m+1} \left(U_{j_1},\ldots,U_{j_m}\right).
\label{eq:compress_U}
\end{eqnarray}
Let $T:=(U_{j_1},\ldots,U_{j_m}) \in \mathcal{X}^m$ denote this compressed subtuple. Let $\mathcal{T}(T)$ be the finite set of all $m$-tuples whose entries are chosen from the finite set $\eta_{m+1}(T)$ (allowing repetitions). For each $V \in \mathcal{T}(T)$, the value $\eta_n(V)$ is a finite set in $\mathcal{F}$. By union boundedness, there exists a set in $\mathcal{F}$ containing $\eta_{m+1}(T)$ and the union of $\{\eta_n(V):V \in \mathcal{T}(T)\}$. Fix any such choice and define it to be $\eta_{n+1}(T)$:
\begin{eqnarray}
\eta_{n+1}(T)\ \in\ \mathcal{F},
\quad
\eta_{m+1}(T)\cup\Big(\bigcup_{V\in\mathcal{T}(T)} \eta_n(V)\Big)\ \subseteq\ \eta_{n+1}(T).
\label{eq:eta_nplus1_def}
\end{eqnarray}
Then, $x_{n+1} \in \eta_{m+1}(T)\subseteq\eta_{n+1}(T)$ by Eq.~(\ref{eq:compress_U}). Moreover, $(x_{i_1},\ldots,x_{i_m}) \in \mathcal{T}(T)$ (its entries lie in $\eta_{m+1}(T)$), so $\eta_n(x_{i_1},\ldots,x_{i_m}) \subseteq \eta_{n+1}(T)$ by Eq.~(\ref{eq:eta_nplus1_def}). Combined with Eq.~(\ref{eq:inductive_cover_n}), this yields $\{x_1,\ldots,x_n\} \subseteq \eta_{n+1}(T)$. This completes the induction and establishes Eq.~(\ref{eq:n_to_m_compression}) for all $n\ge m+1$.

\medskip
\paragraph*{\bf Step 2 (Define an EMX learner from compression).}
Fix an integer $n \ge m+1$ (to be chosen later). Given a sample $S=(x_1,\ldots,x_n) \sim P^n$, define the candidate family
\begin{eqnarray}
\mathcal{H}(S) := \Big\{\eta_n(x_{i_1},\ldots,x_{i_m}) : 1\le i_1<\cdots<i_m\le n\Big\}\ \subseteq\ \mathcal{F}.
\label{eq:candidate_family}
\end{eqnarray}
and the empirical measure
\begin{eqnarray}
\widehat{P}_S(F) := \frac{1}{n}\sum_{t=1}^n \mathbb{1}\{x_t\in F\}.
\label{eq:empirical_measure}
\end{eqnarray}
Let the learner output an empirical maximizer over this (data-dependent) candidate set,
\begin{eqnarray}
G(S) \in \arg\max_{H \in \mathcal{H}(S)} \widehat{P}_S(H).
\label{eq:emx_learner_via_erm_over_candidates}
\end{eqnarray}

Because $P$ is finitely supported, the supremum $\textrm{opt}_{\mathcal{F}}(P)$ is achieved by some $F^\star \in \mathcal{F}$. Let $S^\star:=\{x_t\in S:x_t\in F^\star\}$. Applying Eq.~(\ref{eq:n_to_m_compression}) to the multiset of points in $S^\star$ (padded with repetitions to length $n$ if needed) yields indices $i_1<\cdots<i_m$ such that $S^\star\subseteq \eta_n(x_{i_1},\ldots,x_{i_m})$. Let
\begin{eqnarray}
H^\star:=\eta_n (x_{i_1},\ldots,x_{i_m})\in\mathcal{H}(S).
\end{eqnarray}
Then every sample point that lies in $F^\star$ also lies in $H^\star$, hence $\widehat{P}_S(H^\star)\ge \widehat{P}_S(F^\star)$ and therefore
\begin{eqnarray}
\widehat{P}_S\big(G(S)\big) \ge \widehat{P}_S(H^\star) \ge \widehat{P}_S(F^\star).
\label{eq:empirical_comparison_chain}
\end{eqnarray}

\paragraph*{\bf Step 3 (Generalization via leave-$m$-out concentration).}
The subtlety is that $\mathcal{H}(S)$ depends on $S$. We therefore use a leave-$m$-out argument: each candidate $H_I\in\mathcal{H}(S)$ is determined by at most $m$ sample points and is independent of the remaining $n-m$ points.

For each index set $I=\{i_1<\cdots<i_m\}\subseteq [n]$, define
\begin{eqnarray}
H_I := \eta_n(x_{i_1},\ldots,x_{i_m}) \in \mathcal{H}(S), \quad J_I:=[n] \setminus I,
\end{eqnarray}
and the holdout empirical measure
\begin{eqnarray}
\widehat{P}^{(-I)}_S(H_I) := \frac{1}{n-m}\sum_{t \in J_I}\mathbb{1}\{x_t \in H_I\}.
\label{eq:holdout_empirical_def}
\end{eqnarray}
Conditioned on $\{x_i:i\in I\}$, the variables $\{\mathbb{1}\{x_t\in H_I\}\}_{t\in J_I}$ are i.i.d. Bernoulli with mean $P(H_I)$, so Hoeffding's inequality~\cite{Hoeffding1963} gives, for every fixed $I$,
\begin{eqnarray}
\Pr\left[\abs{\widehat{P}^{(-I)}_S(H_I)-P(H_I)} > \alpha\right] \le 2e^{-2(n-m)\alpha^2}.
\label{eq:hoeffding_holdout}
\end{eqnarray}
Taking a union bound over the $\binom{n}{m}$ index sets yields
\begin{eqnarray}
\Pr\left[\exists\,I \subseteq[n],\, \abs{I}=m: \abs{\widehat{P}^{(-I)}_S(H_I)-P(H_I)} > \alpha\right] \le 2\binom{n}{m}e^{-2(n-m)\alpha^2}.
\label{eq:uniform_over_candidates_holdout}
\end{eqnarray}
On the complement of this event, for every $I$ we have
\begin{eqnarray}
\abs{\widehat{P}_S(H_I)-P(H_I)} &=& \abs{\frac{n-m}{n}\widehat{P}^{(-I)}_S(H_I) + \frac{1}{n}\sum_{t\in I}\mathbb{1}\{x_t \in H_I\}-P(H_I)} \nonumber\\
	&\le& \frac{n-m}{n} \abs{\widehat{P}^{(-I)}_S(H_I)-P(H_I)} + \frac{m}{n} \nonumber \\
	&\le& \alpha + \frac{m}{n}.
\label{eq:full_empirical_from_holdout}
\end{eqnarray}

In addition, applying the one-sided Hoeffding bound to the fixed set $F^\star$ gives
\begin{eqnarray}
\Pr\left[\widehat{P}_S(F^\star) < P(F^\star)-\alpha\right] \le e^{-2n\alpha^2}.
\label{eq:opt_concentration}
\end{eqnarray}
Choose $\alpha:=1/6$ and pick $n$ large enough so that $m/n \le 1/6$ and
\begin{eqnarray}
2\binom{n}{m}e^{-2(n-m)\alpha^2}\ \le\ \frac{1}{6} \quad\mbox{and}\quad e^{-2n\alpha^2} \le \frac{1}{6}.
\label{eq:n_choice_conditions}
\end{eqnarray}
Then, with probability at least $2/3$, we have both
\begin{eqnarray}
\sup_{H \in \mathcal{H}(S)}\abs{\widehat{P}_S(H)-P(H)} \le \alpha + \frac{m}{n} \quad\text{and}\quad \widehat{P}_S(F^\star)\ge P(F^\star)-\alpha.
\end{eqnarray}
On this event, combining Eq.~(\ref{eq:empirical_comparison_chain}) with the deviation bounds yields
\begin{eqnarray}
P\big(G(S)\big) &\ge& \widehat{P}_S\big(G(S)\big)-\left(\alpha+\frac{m}{n}\right) \ge \widehat{P}_S(F^\star)-\left(\alpha+\frac{m}{n}\right) \nonumber\\
	&\ge& P(F^\star)-2\alpha-\frac{m}{n} \nonumber \\
	&\ge& \textrm{opt}_{\mathcal{F}}(P) - \frac{1}{3}.
\label{eq:emx_guarantee_from_compression}
\end{eqnarray}
Thus, $\mathcal{F}$ is $(1/3,1/3)$-EMX learnable.
\end{proof}

\medskip
\begin{proof}[Proof of (i)$\Rightarrow$(ii)]
Assume that $\mathcal{F}$ is $(1/3,1/3)$-EMX learnable. Then, there exists an integer $d$ and a learner $G:\mathcal{X}^d\to\mathcal{F}$ such that for every finitely supported $P$,
\begin{eqnarray}
\Pr_{S\sim P^d}\left[P\big(G(S)\big) \ge \textrm{opt}_{\mathcal{F}}(P)-\frac{1}{3}\right] \ge \frac{2}{3}.
\label{eq:weak_emx_assumption}
\end{eqnarray}
Set
\begin{eqnarray}
m \;:=\; \left\lceil \frac{3d}{2}\right\rceil.
\label{eq:m_choice}
\end{eqnarray}
Define a reconstruction function $\eta:\mathcal{X}^m \to \mathcal{F}$ as follows. For any $m$-tuple $S'=(x_1,\ldots,x_m)$, let $\mathcal{T}(S')$ be the finite collection of all $d$-tuples obtained by selecting $d$ elements from $\{x_1,\ldots,x_m\}$ (allowing repetitions). Consider the finite family
\begin{eqnarray}
\mathcal{U}(S') := \Big\{\{x_1,\ldots,x_m\}\Big\} \cup \Big\{G(T): T \in \mathcal{T}(S') \Big\} \subseteq \mathcal{F}.
\label{eq:U_family}
\end{eqnarray}
Since $\mathcal{F}$ is union bounded and every set in $\mathcal{U}(S')$ is finite, there exists $\eta(S')\in\mathcal{F}$ such that
\begin{eqnarray}
\bigcup_{U\in\mathcal{U}(S')} U \ \subseteq\ \eta(S').
\label{eq:eta_contains_union}
\end{eqnarray}
We claim that this $\eta$ is an $(m+1)\to m$ monotone compression scheme. Let $F\in\mathcal{F}$ and let $S=(x_1,\ldots,x_{m+1})$ be a sequence with all entries in $F$. Suppose towards a contradiction that for every $t\in\{1,\ldots,m+1\}$, writing $S^{(-t)}$ for the $m$-tuple obtained by removing $x_t$, one has
\begin{eqnarray}
x_t \notin \eta\big(S^{(-t)}\big).
\label{eq:xt_not_in_reconstruction}
\end{eqnarray}
Fix any $d$-tuple $T$ drawn from $\{x_1,\ldots,x_{m+1}\}$. If $x_t \notin T$, then $T$ is a $d$-tuple drawn from $S^{(-t)}$, so Eq.~(\ref{eq:eta_contains_union}) implies $G(T) \subseteq \eta(S^{(-t)})$ and hence $x_t \notin G(T)$. Therefore, for every such $T$,
\begin{eqnarray}
G(T)\cap \{x_1,\ldots,x_{m+1}\} \subseteq \{ \mbox{entries of $T$} \}.
\label{eq:G_T_support_on_T}
\end{eqnarray}
Let $P$ be the uniform distribution on $\{x_1,\ldots,x_{m+1}\}$,
\begin{eqnarray}
P(\{x_t\}) = \frac{1}{m+1}, \quad t=1,\ldots,m+1.
\label{eq:uniform_distribution_on_S}
\end{eqnarray}
Since $S\subseteq F$, we have $\textrm{opt}_{\mathcal{F}}(P)=1$. But Eq.~(\ref{eq:G_T_support_on_T}) implies $P(G(T)) \le d/(m+1)$ for every $T\sim P^d$, and by Eq.~(\ref{eq:m_choice}), we have $d/(m+1)<2/3$. Hence,
\begin{eqnarray}
\Pr_{T\sim P^d}\!\left[P(G(T)) \ge \frac{2}{3}\right] = 0,
\label{eq:zero_probability}
\end{eqnarray}
contradicting Eq.~(\ref{eq:weak_emx_assumption}) under $P$. Therefore, there exists some $t$ for which $x_t\in \eta(S^{(-t)})$, i.e., there exists an $m$-subsequence $S'$ with $\{x_1,\ldots,x_{m+1}\} \subseteq \eta(S')$. This proves the existence of an $(m+1) \to m$ monotone compression scheme.
\end{proof}

{\bf Lemma~\ref{lem:learnability_compression_equivalence}} formalizes an important conceptual message: for union bounded families of finite sets, the existence of a learner achieving even a coarse approximation to the maximum measure is equivalent to the ability to discard a single sample point while retaining enough information to reconstruct a finite superset containing the entire original sample.

\subsection{Cardinalities and the emergence of set-theoretic independence}

The next task is to relate the existence of monotone compression schemes to the cardinality of the domain. For a set $\mathcal{X}$, define the family of all finite subsets of $\mathcal{X}$,
\begin{eqnarray}
\mathcal{F}_{\textrm{fin}}^{\mathcal{X}} := \big\{F \subseteq \mathcal{X} : \mbox{$F$ is finite}\big\}.
\label{eq:finite_subsets_class}
\end{eqnarray}
This family is union bounded. The key theorem identifies precisely when $\mathcal{F}_{\textrm{fin}}^{\mathcal{X}}$ admits a ``constant-size'' monotone compression scheme.

We briefly recall the notation for infinite cardinals. Let $\aleph_0$ denote the cardinality of $\mathbb{N}$, and for each $k \ge 0$ let $\aleph_{k+1}$ be the smallest cardinal strictly larger than $\aleph_k$. For a set $\mathcal{X}$, the statement $\abs{\mathcal{X}} \le \aleph_k$ means that $\mathcal{X}$ injects into a set of cardinality $\aleph_k$.

\begin{theorem}[Compression and cardinality for finite-subset classes]
\label{thm:compression_cardinality}
Fix an integer $k \ge 0$ and a set $\mathcal{X}$. Then $\abs{\mathcal{X}} \le \aleph_k$ if and only if $\mathcal{F}_{\textrm{fin}}^{\mathcal{X}}$ admits a $(k+2) \to (k+1)$ monotone compression scheme.
\end{theorem}

\begin{proof}[Proof of ($\Rightarrow$)]
Assume $\abs{\mathcal{X}} \le \aleph_k$. By the well-ordering theorem (equivalent to the axiom of choice), there exists a well-order $\prec_k$ on $\mathcal{X}$ whose order type is the initial ordinal $\omega_k$ of cardinality $\aleph_k$~\cite{Jech2003,Kunen1980}. In particular, for any $x \in \mathcal{X}$, the initial segment
\begin{eqnarray}
I_k(x) := \{y \in \mathcal{X}: y \prec_k x\}
\label{eq:initial_segment}
\end{eqnarray}
has cardinality strictly smaller than $\aleph_k$, and since $k$ is finite this implies $\abs{I_k(x)} \le \aleph_{k-1}$ for $k \ge 1$ (and $\abs{I_0(x)} < \aleph_0$ means $I_0(x)$ is finite).

We define a $(k+2)\to(k+1)$ compression scheme for $\mathcal{F}_{\textrm{fin}}^{\mathcal{X}}$. Let $S=\{x_0,\ldots,x_{k+1}\} \subseteq \mathcal{X}$ be a set of size $k+2$. We construct $k+1$ ``representatives'' $z_k,\ldots,z_0$ recursively. Let $z_k$ be the $\prec_k$-maximum of $S$. For $k \ge 1$, the remaining points $S \setminus \{z_k\}$ lie inside $I_k(z_k)$, whose cardinality is at most $\aleph_{k-1}$, and hence admits a well-order $\prec_{k-1}$ of type $\omega_{k-1}$. Let $z_{k-1}$ be the $\prec_{k-1}$-maximum of $S \setminus \{z_k\}$. Continuing recursively, we obtain $z_k,z_{k-1},\ldots,z_1$, and at the final stage we are left with two points $u,v \in I_1(z_1)$.

The crucial point is that $I_1(z_1)$ is countable (cardinality $\aleph_0$), so we may choose $\prec_0$ on $I_1(z_1)$ of order type $\omega$. Let $z_0$ be the $\prec_0$-maximum of the two points $\{u,v\}$. Then the other point in $\{u,v\}$ must lie in the finite initial segment $I_0(z_0):=\{y\in I_1(z_1): y\prec_0 z_0\}$.

We then define the compressed representation of $S$ to be the $(k+1)$-element set
\begin{eqnarray}
S' := \{z_k,z_{k-1},\ldots,z_0\},
\label{eq:compressed_set_Sprime}
\end{eqnarray}
and define the reconstruction function
\begin{eqnarray}
\eta(S') := S' \cup I_0(z_0).
\label{eq:reconstruction_cardinality}
\end{eqnarray}
By construction, $\eta(S')$ is finite and hence belongs to $\mathcal{F}_{\textrm{fin}}^{\mathcal{X}}$. Moreover, $\eta(S')$ contains $S'$ and also contains the remaining point among $\{u,v\}$ via inclusion in $I_0(z_0)$. Since all other points of $S$ are among the $z_i$, we conclude that $S \subseteq \eta(S')$. This proves the existence of a $(k+2) \to (k+1)$ monotone compression scheme.
\end{proof}

\begin{lemma}
\label{lem:cardinality_descent}
Let $r \ge 1$ be an integer and let $\mathcal{Y} \subset \mathcal{X}$ be infinite sets with $\abs{\mathcal{Y}} < \abs{\mathcal{X}}$. If $\mathcal{F}_{\textrm{fin}}^{\mathcal{X}}$ admits an $(r+1) \to r$ monotone compression scheme, then $\mathcal{F}_{\textrm{fin}}^{\mathcal{Y}}$ admits an $r \to (r-1)$ monotone compression scheme.
\end{lemma}

\begin{proof}
Let $\eta:\mathcal{X}^r \to \mathcal{F}_{\textrm{fin}}^{\mathcal{X}}$ be a reconstruction function witnessing an $(r+1) \to r$ monotone compression scheme for $\mathcal{F}_{\textrm{fin}}^{\mathcal{X}}$. Consider the set
\begin{eqnarray}
Z := \bigcup_{\substack{T \subseteq \mathcal{Y} \\ \abs{T} \le r}} \eta(T),
\label{eq:Z_union}
\end{eqnarray}
where we identify a finite subset $T=\{t_1,\ldots,t_s\}$ with an $r$-tuple by padding with repetitions as needed to form an element of $\mathcal{X}^r$. Each $\eta(T)$ is finite, and the union ranges over at most $\abs{\mathcal{Y}}$ many subsets. Since $\mathcal{Y}$ is infinite, the union of $\abs{\mathcal{Y}}$ many finite sets has cardinality at most $\abs{\mathcal{Y}}$, hence $\abs{Z} \le \abs{\mathcal{Y}}$. Because $\abs{\mathcal{Y}} < \abs{\mathcal{X}}$, there exists a point $x \in \mathcal{X} \setminus Z$.

Fix any $r$-element subset $T \subseteq \mathcal{Y}$ and define $S:=T \cup \{x\}$, which has size $r+1$. By the $(r+1) \to r$ compression property, there exists an $r$-subset $S' \subseteq S$ such that $S \subseteq \eta(S')$. We claim that $x \in S'$. Indeed, if $x\notin S'$, then $S' \subseteq \mathcal{Y}$, hence $\eta(S') \subseteq Z$ by definition of $Z$, contradicting $x\notin Z$ while $x \in S \subseteq \eta(S')$. Therefore, $x \in S'$, and letting $U:=S' \setminus \{x\}$ we have $\abs{U} \le r-1$ and
\begin{eqnarray}
T \subseteq \eta\big(U\cup\{x\}\big).
\label{eq:T_in_eta}
\end{eqnarray}
Define the reconstruction function $\eta_{\mathcal{Y}}:\mathcal{Y}^{r-1} \to \mathcal{F}_{\textrm{fin}}^{\mathcal{Y}}$ by
\begin{eqnarray}
\eta_{\mathcal{Y}}(U) := \eta\!\big(U\cup\{x\}\big)\cap \mathcal{Y}.
\label{eq:eta_Y_def}
\end{eqnarray}
Then, Eq.~(\ref{eq:T_in_eta}) implies $T \subseteq \eta_{\mathcal{Y}}(U)$. This establishes an $r \to (r-1)$ monotone compression scheme for $\mathcal{F}_{\textrm{fin}}^{\mathcal{Y}}$.
\end{proof}

\begin{proof}[Proof of ($\Leftarrow$)]
Assume that $\mathcal{F}_{\textrm{fin}}^{\mathcal{X}}$ admits a $(k+2) \to (k+1)$ monotone compression scheme with reconstruction $\eta:\mathcal{X}^{k+1} \to \mathcal{F}_{\textrm{fin}}^{\mathcal{X}}$. We show that this forces $\abs{\mathcal{X}} \le \aleph_k$.

Suppose towards a contradiction that $\abs{\mathcal{X}} > \aleph_k$. Then $\mathcal{X}$ contains a subset of cardinality $\aleph_{k+1}$; replacing $\mathcal{X}$ by such a subset does not affect the existence of a compression scheme for $\mathcal{F}_{\textrm{fin}}^{\mathcal{X}}$ because the restriction map preserves the property. Applying {\bf Lemma~\ref{lem:cardinality_descent}} iteratively along the strict chain
\begin{eqnarray}
\aleph_{k+1} > \aleph_k > \cdots > \aleph_0
\label{eq:aleph_chain}
\end{eqnarray}
yields, after $k+1$ steps, a $1 \to 0$ monotone compression scheme for $\mathcal{F}_{\textrm{fin}}^{\mathcal{Y}}$ for some infinite $\mathcal{Y}$ (in fact $\abs{\mathcal{Y}}=\aleph_0$). But no infinite set admits a $1 \to 0$ monotone compression scheme: if $\eta(\emptyset)$ were a finite reconstruction for the empty compressed sample, the monotone condition applied to each singleton $\{y\}$ would force $y \in \eta(\emptyset)$ for all $y \in \mathcal{Y}$, hence $\eta(\emptyset) \supseteq \mathcal{Y}$, contradicting finiteness. Therefore, $\abs{\mathcal{X}} \le \aleph_k$.
\end{proof}

{\bf Theorem~\ref{thm:compression_cardinality}} is precisely the point at which set theory enters the learning-theoretic narrative. Specializing to $\mathcal{X}=[0,1]$ and $\mathcal{F}=\mathcal{F}_{\textrm{fin}}^{[0,1]}$, {\bf Lemma~\ref{lem:learnability_compression_equivalence}} and {\bf Theorem~\ref{thm:compression_cardinality}} together show that weak EMX learnability is equivalent to a cardinal statement about the continuum: $\mathcal{F}_{\textrm{fin}}^{[0,1]}$ is $(1/3, 1/3)$-EMX learnable if and only if $\abs{[0,1]} \le \aleph_k$ for some finite $k$. In models of set theory where the continuum hypothesis holds, $\abs{[0,1]}=\aleph_1$ and thus a $3 \to 2$ monotone compression exists, implying EMX learnability. In models where the continuum is forced to be larger than $\aleph_k$ for every finite $k$, no such compression exists, implying non-learnability. Since the existence of such models is Bconsistent with ZFC by classical independence results~\cite{Cohen1964,Jech2003}, the learnability of this seemingly elementary class becomes independent of ZFC, as established in Ref.~\cite{BenDavid2019}.

\section{From Mathematical Learnability to Physical Learnability}

The independence phenomenon (reviewed in the previous section) hinges on a methodological choice that is ubiquitous in learning theory: a learner is formalized as a set-theoretic object---a function (or randomized kernel) from finite samples to hypotheses---and learnability is defined by quantifying over the existence of such objects. This perspective is powerful precisely because it separates statistical questions from implementation. Yet, as the EMX example shows, the separation can come at a foundational cost: when the relevant objects live over genuinely infinite domains, the statement ``there exists a learner'' may no longer be a determinate mathematical fact within ZFC~\cite{BenDavid2019}. Thus, the aim of this section is to introduce a minimal formal language for incorporating \emph{admissibility} constraints into learnability. The point is not to replace statistical learning theory by physics, but to isolate the precise location where physical constraints may, and arguably should, enter: in the specification of the allowable class of inference procedures.

\subsection{Admissible learners and relative notions of learnability}

Fix a domain $\mathcal{X}$ and a hypothesis family $\mathcal{F} \subseteq \{0,1\}^{\mathcal{X}}$. In the standard set-theoretic formulation, an EMX learner is any map of the form
\begin{eqnarray}
G:\bigcup_{n \in\mathbb{N}}\mathcal{X}^n \to \mathcal{F},
\label{eq:general_learner_map}
\end{eqnarray}
and learnability is defined by the existence of such a map satisfying the EMX guarantee for all distributions in a prescribed class (here, finitely supported distributions). Operationally, however, a learner is not an arbitrary function, but it is a physical process. This motivates the following definition, which makes the admissible class explicit.

\begin{definition}[Admissible learner class]
\label{def:admissible_class}
An admissible learner class for $(\mathcal{X}, \mathcal{F})$ is a set $\mathfrak{L}$ of learners of the form in Eq.~(\ref{eq:general_learner_map}). We interpret $G \in \mathfrak{L}$ as ``$G$ is implementable'' under the chosen computational/physical model. In particular, $\mathfrak{L}$ may represent deterministic algorithms, randomized procedures, quantum channels followed by measurements, or any other operationally meaningful family of inference rules~\cite{Bang2025SecurePAC}.
\end{definition}

The associated notion of learnability is simply the standard one with the existential quantifier restricted to $\mathfrak{L}$.

\begin{definition}[$\mathfrak{L}$-EMX learnability]
\label{def:L_emx_learnability}
Let $\mathcal{P}$ be a class of (finitely supported) distributions over $\mathcal{X}$ and fix $\epsilon, \delta \in (0,1)$. We say that $\mathcal{F}$ is $(\epsilon,\delta)$-EMX learnable over $\mathcal{P}$ relative to $\mathfrak{L}$ if there exists $G \in \mathfrak{L}$ and an integer $d=d(\epsilon,\delta)$ such that for every $P \in \mathcal{P}$,
\begin{eqnarray}
\Pr_{S\sim P^d}\left[ P\big(G(S)\big) \ge \textrm{opt}_{\mathcal{F}}(P)-\epsilon \right] \ge 1-\delta,
\label{eq:L_emx_definition}
\end{eqnarray}
where $\textrm{opt}_{\mathcal{F}}(P)$ is as in Eq.~(\ref{eq:opt_emx}).
\end{definition}

The relative learnability is monotone with respect to admissibility, a trivial but conceptually useful fact: making the physical model more permissive cannot destroy learnability, whereas tightening admissibility can.

\begin{theorem}[Monotonicity under admissibility]
\label{thm:monotonicity_admissibility}
Let $\mathfrak{L}_1 \subseteq \mathfrak{L}_2$ be admissible learner classes for $(\mathcal{X},\mathcal{F})$. If $\mathcal{F}$ is $(\epsilon,\delta)$-EMX learnable over $\mathcal{P}$ relative to $\mathfrak{L}_1$, then it is $(\epsilon,\delta)$-EMX learnable over $\mathcal{P}$ relative to $\mathfrak{L}_2$.
\end{theorem}

\begin{proof}
If $G \in \mathfrak{L}_1$ satisfies Eq.~(\ref{eq:L_emx_definition}), then $G\in\mathfrak{L}_2$ as well, so the same sample size $d(\epsilon,\delta)$ witnesses learnability relative to $\mathfrak{L}_2$.
\end{proof}

The ``undecidability of learnability'' phenomenon of Ref.~\cite{BenDavid2019} should therefore be read as a statement about a particular (maximally permissive) choice of $\mathfrak{L}$, namely the class of all set-theoretic functions of the form in Eq.~(\ref{eq:general_learner_map}). Once $\mathfrak{L}$ is fixed by operational constraints, the object whose existence is being asserted changes, and with it the logical status of the learnability claim.

The existing learning theory deliberately treats the learner as unconstrained in order to isolate statistical structure. Our viewpoint keeps that statistical objective intact, but insists that the quantifier ``$\exists\,G$'' must be read relative to an explicit admissibility model. In other words, ``learnability'' is not an absolute predicate of $(\mathcal{X},\mathcal{F},\mathcal{P})$ alone; it is a predicate of the pair (\emph{learning problem, admissible inference physics}). When $\mathfrak{L}$ is taken to be the maximal set of all functions, one recovers the standard set-theoretic notion and therefore inherits its logical subtleties. When $\mathfrak{L}$ is fixed by operational constraints (finite description, quantum access, relativistic causality), learnability becomes a statement about the feasibility of a physically meaningful class of procedures, and the meaning of ``there exists a learner'' becomes aligned with experimental semantics.

\subsection{Finite description and the countable collapse}

A fundamental operational constraint is the finite description. In any laboratory, a ``sample'' is ultimately stored, processed, and communicated through finite physical resources: finite memory, finite time, and finite-precision interfaces. Even when the underlying variable is mathematically modeled as continuum-valued (for example, $\mathcal{X}=[0,1]$), what the learner actually receives is a finite record: a finite binary string produced by an instrument (digits on a display, a timestamped detector click, a binned measurement outcome), and any hypothesis it outputs must itself be communicated as a finite record.

This immediately raises a tension with the fully set-theoretic formulation of a learner as an arbitrary function on $\bigcup_n \mathcal{X}^n$ when $\mathcal{X}$ is uncountable. The formal learner may depend on distinctions between points of $\mathcal{X}$ that no finite record can ever reveal, and it may output hypotheses whose ``names'' require non-finitary information about $\mathcal{X}$. To express, in minimal mathematical terms, the idea that both inputs and outputs are finitely describable, one wants the physically distinguishable domain to admit a countable codebook. A minimal mathematical avatar of ``finite description'' is, therefore, \emph{countability}. Concretely, a countable $\mathcal{X}$ can be injected into $\mathbb{N}$ and hence represented by finite strings, while an uncountable $\mathcal{X}$ cannot be exhaustively named by any finitary codebook. The point is that this seemingly small change already collapses the set-theoretic pathology of the Ben-David example into an explicit and ZFC-provable learnability statement.

To connect this observation to the Ben-David phenomenon, we consider the canonical hypothesis class $\mathcal{F}_{\textrm{fin}}^{\mathcal{X}}$ of all finite subsets of $\mathcal{X}$. In the fully set-theoretic model, {\bf Theorem~\ref{thm:compression_cardinality}} shows that constant-size monotone compression for $\mathcal{F}_{\textrm{fin}}^{\mathcal{X}}$ is equivalent to the cardinal bound $\abs{\mathcal{X}} \le \aleph_k$ for some finite $k$, which is precisely where independence enters when $\mathcal{X}=[0,1]$. Operationally, however, once $\mathcal{X}$ is countable (or once only a countable subset is physically distinguishable), the strongest nontrivial compression exists outright and can be implemented by an explicit finitary rule.

\begin{theorem}[A computable $2 \to 1$ monotone compression for countable domains]
\label{thm:countable_two_to_one}
Let $\mathcal{X}$ be a countable set, and fix an injection $\idx:\mathcal{X} \to \mathbb{N}$. Then, the class $\mathcal{F}_{\textrm{fin}}^{\mathcal{X}}$ admits a $2 \to 1$ monotone compression scheme with reconstruction function
\begin{eqnarray}
\eta(x) := \{y \in \mathcal{X}: \idx(y) \le \idx(x)\}.
\label{eq:eta_countable}
\end{eqnarray}
Moreover, if $\idx$ is given as part of the model, the map $x \mapsto \eta(x)$ is finitary and effectively describable.
\end{theorem}

\begin{proof}
The injection $\idx$ induces a total order on $\mathcal{X}$ by declaring $y \preceq x$ whenever $\idx(y) \le \idx(x)$. The reconstruction map Eq.~(\ref{eq:eta_countable}) is simply the ``initial segment'' of this order up to $x$. Fix any $F \in \mathcal{F}_{\textrm{fin}}^{\mathcal{X}}$ and any two points $x_1,x_2 \in F$. Let $x_{\max}$ be the point among $\{x_1,x_2\}$ maximizing $\idx(\cdot)$ (equivalently, the $\preceq$-maximum). The compression rule sends $x_{\max}$. It remains to verify that $\eta(x_{\max})$ is a valid monotone reconstruction. First, $\eta(x_{\max})$ is finite: because $\idx$ is injective, the preimage $\idx^{-1}(\{0,1,\dots,\idx(x_{\max})\})$ contains at most one element per integer, hence has size at most $\idx(x_{\max})+1$. Second, $\eta(x_{\max})$ contains both sample points: by definition $\idx(x_1) \le \idx(x_{\max})$ and $\idx(x_2) \le \idx(x_{\max})$, so $x_1,x_2 \in \eta(x_{\max})$. Therefore,
\begin{eqnarray}
\{x_1,x_2\} \subseteq \eta(x_{\max}),
\label{eq:two_to_one_property}
\end{eqnarray}
which is exactly the $2 \to 1$ monotone compression condition for $\mathcal{F}_{\textrm{fin}}^{\mathcal{X}}$.
\end{proof}

{\bf Theorem~\ref{thm:countable_two_to_one}} is intentionally elementary, but its role is foundational: it shows that once the domain admits an explicit naming/encoding, the weakest monotone compression (dropping just one point) is not only possible but constructive. In particular, it provides a concrete finitary witness that replaces the set-theoretic ``there exists'' in the unconstrained formulation.

By combining {\bf Theorem~\ref{thm:countable_two_to_one}} with {\bf Lemma~\ref{lem:learnability_compression_equivalence}}, we can yield an immediate consequence: the canonical ``finite-subset'' EMX task is provably learnable once the domain is constrained to be countable, and the learner can be taken to be explicitly definable from an indexing of the domain.

\begin{corollary}[Weak EMX learnability for $\mathcal{F}_{\textrm{fin}}^{\mathcal{X}}$ on countable domains]
\label{cor:countable_emx}
If $\mathcal{X}$ is countable, then $\mathcal{F}_{\textrm{fin}}^{\mathcal{X}}$ is $(1/3, 1/3)$-EMX learnable over the class of finitely supported distributions on $\mathcal{X}$. In particular, this learnability statement is provable in ZFC.
\end{corollary}

\begin{proof}
By {\bf Theorem~\ref{thm:countable_two_to_one}}, $\mathcal{F}_{\textrm{fin}}^{\mathcal{X}}$ admits a $(m+1)\to m$ monotone compression scheme with $m=1$. The family $\mathcal{F}_{\textrm{fin}}^{\mathcal{X}}$ is union bounded (finite unions of finite sets are finite), so {\bf Lemma~\ref{lem:learnability_compression_equivalence}} applies and yields $(1/3,1/3)$-EMX learnability.
\end{proof}

{\bf Corollary~\ref{cor:countable_emx}} does not contradict the independence phenomenon for $\mathcal{X}=[0,1]$. Rather, it sharpens its operational interpretation. The set-theoretic undecidability is driven by two idealizations: (i) treating $\mathcal{X}$ as a continuum of perfectly distinguishable points, and (ii) allowing the learner to be an arbitrary function on that continuum. If instead the learner only ever interacts with $\mathcal{X}$ through a finite-description interface---an encoding, a rounding rule, a coarse-graining, or any other physical measurement map---then the effective domain seen by the learner is countable, and the corresponding EMX instance becomes learnable in a direct and provable way. In this sense, the countability is not a mere set-theoretic curiosity: it is the minimal mathematical proxy for the empirical fact that data and hypotheses are finitely representable.

\subsection{Quantum and relativistic constraints as restrictions on admissible inference}

Countability is only one facet of physical admissibility. When data are quantum or distributed across spacelike separated parties, the admissible class $\mathfrak{L}$ is restricted not merely by description length, but by fundamental information-theoretic principles. Two paradigmatic examples are the no-cloning and no-signaling principles. In the perspective of this study, these principles matter because they constrain how samples may be accessed and combined, and therefore constrain the set of admissible inference procedures whose existence is quantified in learnability definitions.

\subsubsection{No-cloning} 

In quantum settings, a ``sample'' may be a quantum state $\hat{\rho}$ on a Hilbert space $\mathcal{H}$. A learning procedure can interact with a bounded number of copies $\hat{\rho}^{\otimes d}$ via physically allowed operations (quantum channels, measurements), but it cannot in general manufacture additional independent copies from a single instance. This limitation is formalized by the no-cloning theorem~\cite{WoottersZurek1982,Dieks1982}.

\begin{theorem}[No-cloning]
\label{thm:no_cloning}
Let $\mathcal{H}$ be a Hilbert space with $\dim(\mathcal{H}) \ge 2$, and fix a unit vector $\ket{0} \in \mathcal{H}$. There is no linear isometry $\hat{U}:\mathcal{H}\otimes\mathcal{H} \to \mathcal{H}\otimes\mathcal{H}$ such that
\begin{eqnarray}
\hat{U}\big(\ket{\psi} \otimes \ket{0}\big) = \ket{\psi} \otimes \ket{\psi}
\label{eq:universal_cloner}
\end{eqnarray}
for all quantum states $\ket{\psi} \in \mathcal{H}$.
\end{theorem}

\begin{proof}[Proof sketch]
Assume such an isometry $\hat{U}$ existed. Choose two distinct non-orthogonal unit vectors $\ket{\psi}, \ket{\phi}$ with $0 < \abs{\braket{\psi}{\phi}} < 1$. The preservation of inner products under $\hat{U}$ would imply $\braket{\psi}{\phi} = \braket{\psi}{\phi}^2$, forcing $\abs{\braket{\psi}{\phi}} \in \{0,1\}$, a contradiction (See appendix~\ref{appendix:nocloning} for complete proof).
\end{proof}

The no-cloning turns the ``number of samples'' into a physical resource: one cannot freely replicate an unknown training example to amplify information. As a result, quantum learning problems naturally trade sample complexity for copy complexity. Moreover, admissibility is no longer ``any function of the data'' but ``any quantum protocol acting on the provided copies and outputting a classical description.'' This restriction of $\mathfrak{L}$ is not a computational afterthought; it changes what inference is even defined to mean for quantum data, and it leads to quantitative lower bounds (e.g., for state discrimination) that have no classical analogue.

\subsubsection{No-signaling} 

In distributed or relativistic settings, the admissibility is constrained by the causal structure of information flow. A minimal formal constraint is no-signaling~\cite{PopescuRohrlich1994,Barrett2005}: local choices of operations at one location cannot instantaneously affect the statistics of outcomes at a spacelike separated location. In quantum theory this property is not an additional postulate but a consequence of the tensor-product structure and the completeness of measurements.

\begin{definition}[No-signaling correlations]
\label{def:no_signaling}
Let $X, Y$ be sets of inputs and $A, B$ sets of outputs. A conditional distribution $p(a,b | x,y)$ is no-signaling if for all $x,x' \in X$ and all $y,y' \in Y$,
\begin{eqnarray}
\sum_{a \in A} p(a,b | x,y) = \sum_{a \in A} p(a,b | x',y) \quad (\forall \, b \in B),
\label{eq:no_signal_A_to_B}
\end{eqnarray}
and
\begin{eqnarray}
\sum_{b \in B} p(a,b | x,y) = \sum_{b \in B} p(a,b | x,y') \quad (\forall \, a \in A).
\label{eq:no_signal_B_to_A}
\end{eqnarray}
\end{definition}

\begin{theorem}[Quantum mechanics is no-signaling~\cite{NielsenChuang2010}]
\label{thm:quantum_no_signaling}
Let $\hat{\rho}_{AB}$ be a bipartite quantum state on $\mathcal{H}_A\otimes\mathcal{H}_B$. For each input $x \in X$ let $\{\hat{M}_a^x\}_{a \in A}$ be a POVM on $\mathcal{H}_A$, and for each $y \in Y$ let $\{\hat{N}_b^y\}_{b \in B}$ be a POVM on $\mathcal{H}_B$. Define
\begin{eqnarray}
p(a,b|x,y) := \tr{(\hat{M}_a^x \otimes \hat{N}_b^y) \hat{\rho}_{AB}}.
\label{eq:quantum_correlation}
\end{eqnarray}
Then, $p(a,b | x,y)$ is no-signaling in the sense of {\bf Definition~\ref{def:no_signaling}}.
\end{theorem}

\begin{proof}[Proof sketch]
Sum Eq.~(\ref{eq:quantum_correlation}) over $a$ and use POVM completeness $\sum_a \hat{M}_a^x=I_A$ to obtain
\begin{eqnarray}
\sum_{a \in A} p(a,b | x,y) = \tr{(I_A \otimes \hat{N}_b^y) \hat{\rho}_{AB}},
\end{eqnarray}
which is independent of $x$. The other marginal is analogous. Full proof is provided in appendix~\ref{appendix:nosignaling}.
\end{proof}

The no-signaling constrains what distributed inference protocols can do without communication: the local output statistics available to one party cannot depend on the other party's input choice. In our framework, the admissible learner behaviors in relativistic/distributed scenarios must lie in a constraint set carved out by Eq.~(\ref{eq:no_signal_A_to_B}) and Eq.~(\ref{eq:no_signal_B_to_A}). When the relevant input/output alphabets are finite, these constraints are linear and define a polytope; consequently, the existence of a protocol meeting a target success criterion becomes a concrete feasibility question over an explicitly described set.

Taken together, the countable collapse of {\bf Corollary~\ref{cor:countable_emx}} and the structural constraints of {\bf Theorem~\ref{thm:no_cloning}} and {\bf Theorem~\ref{thm:quantum_no_signaling}} motivate a general viewpoint. The set-theoretic notion of learnability treats ``learner existence'' as a purely mathematical fact about unconstrained functions. Operationally, however, learnability is a property of what can be achieved by admissible processes under concrete information constraints. Making the admissible class explicit provides a principled bridge between statistical learning theory and the physics of information, and it clarifies which aspects of the Ben-David undecidability phenomenon are artifacts of non-operational idealizations and which reflect deeper limitations that persist under physically meaningful restrictions.

\section{A Proposed Framework: Physics-Aware Learnability (PL)}

The preceding sections highlight a tension that is easy to overlook when the learnability is treated purely as a property of set-theoretic functions. On the one hand, the abstraction ``a learner is an arbitrary map from samples to hypotheses'' enables powerful structural results and isolates statistical phenomena. On the other hand, the EMX independence phenomenon shows that this abstraction can turn learnability into a statement whose truth may depend on set-theoretic axioms~\cite{BenDavid2019}. The natural response is not to abandon abstraction, but to place the abstraction at the correct boundary: the definition of learnability should quantify over admissible inference procedures, and the admissibility should be specified by an explicit model of physical access to data. Here, we formalize this idea as \emph{physics-aware learnability} (PL) and organizes it into a sequence of operational steps (Fig.~\ref{fig:pl_framework}): we first define a learning task in a theory-agnostic way, then specify an admissible family of physical protocols, and only then ask whether some admissible protocol can achieve near-optimal utility with high probability.

\begin{figure}[t]
\centering
\includegraphics[width=\linewidth]{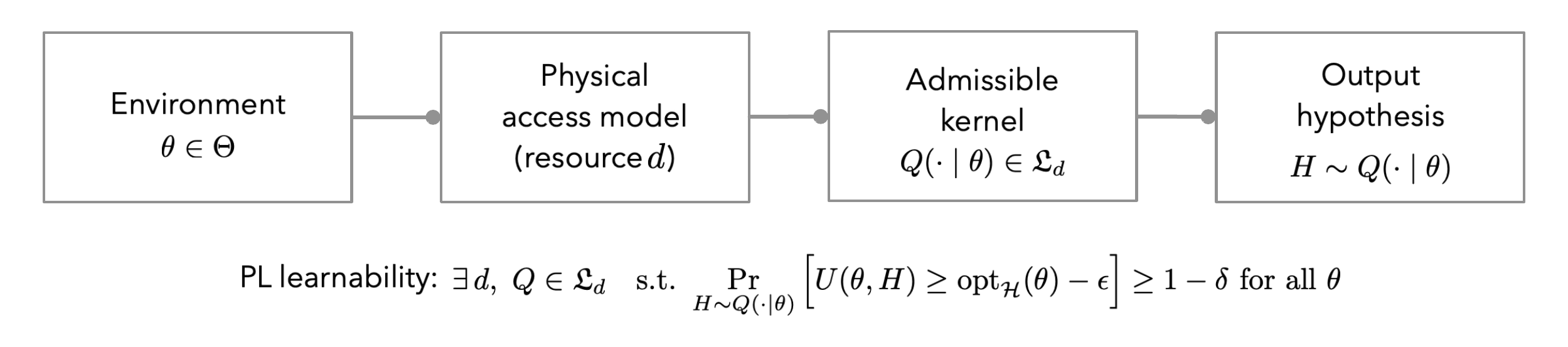}
\caption{Schematic of physics-aware learnability (PL). A learning task $(\Theta,\mathcal{H},U)$ is paired with a physically specified access model that determines, for each sample budget $d$, a set $\mathfrak{L}_d$ of admissible conditional output laws $Q(\cdot \mid \theta)$. PL learnability asks whether some $Q \in \mathfrak{L}_d$ achieves near-optimal utility with high probability uniformly over environments.}
\label{fig:pl_framework}
\end{figure}

The guiding intuition is simple, but its consequences are easy to miss if one starts from purely set-theoretic learners. In an actual experiment, ``data'' are not abstract points in a set, but physical systems prepared by an unknown source and delivered through an interface with finite resources. A learning procedure is itself a physical process. Specifically, it interrogates the available systems subject to the dynamical rules of a background theory, and it produces a finite classical record---a bit string that names the final hypothesis. This is where the operational boundary lies. Beyond that boundary, the set theory may endow us with arbitrary functions on continua. However, on the laboratory side, only those input--output behaviors that can arise from the admissible physical interactions should count as ``learners''. Put differently, the learnability should not be a statement about the existence of a function in a mathematical universe. Instead, it should be a statement about the existence of a protocol consistent with a specified access model.

Once one takes this boundary seriously, several familiar ``background assumptions'' become structural. In quantum experiments, the no-cloning forbids the routine classical idealization that one can freely duplicate an unknown sample; the measurement back-action constrains what can be inferred from finitely many copies; and in distributed settings, the no-signaling constrains which correlations can be generated without communication. Even in fully classical laboratories, finite precision and finite memory mean that a continuum-valued variable is never accessed directly: the interface returns a coarse-grained label. PL elevates all such constraints from informal caveats to an explicit part of the learning model.

\subsection{Operational primitives: environments, hypotheses, and admissible protocols}

A useful slogan is: before asking what algorithm we will run, we should decide what we will call a success. In PL, the notion of success is encoded in a task, while the physically meaningful notion of what we are allowed to do is encoded in an access model. Separating these two ingredients is what ultimately lets us compare classical, quantum, and coarse-grained settings within one common language.

\begin{definition}[Learning task]
\label{def:task}
A learning task is a triple $(\Theta, \mathcal{H}, U)$ where $\Theta$ is a set of environments, $\mathcal{H}$ is a set of hypotheses, and
\begin{eqnarray}
U:\Theta \times \mathcal{H} \to [0,1]
\label{eq:utility}
\end{eqnarray}
is a utility (or reward) functional. For $\theta \in \Theta$, define the optimal value
\begin{eqnarray}
\textrm{opt}_{\mathcal{H}}(\theta) := \sup_{h \in \mathcal{H}} U(\theta,h).
\label{eq:opt_task}
\end{eqnarray}
\end{definition}

The three components play conceptually distinct roles. The environment parameter $\theta \in \Theta$ is the unknown ``state of the world'' one aims to learn from (a probability distribution in classical statistics, a quantum state or channel in quantum settings, a no-signaling box in distributed scenarios, and so on). The hypothesis $h \in\mathcal{H}$ is what the learner is allowed to output as its final explanation/decision (a classifier, a regressor, a finite set, a quantum label, etc.). Finally, the utility $U(\theta,h)$ specifies what we will optimize: it is the operational score by which we will judge whether the learner did well in environment $\theta$ when it outputs $h$. Normalizing to $[0,1]$ is convenient and entails no loss of generality (any bounded performance metric can be rescaled). Crucially, {\bf Definition~\ref{def:task}} does not say how $h$ is obtained from data; it only defines what would count as near-optimal performance if we had the power to choose any hypothesis in $\mathcal{H}$.

To avoid importing set-theoretic idealizations through the output space, we explicitly encode the physical fact that a learner produces a finite classical description. It is therefore natural to treat the hypothesis space as representable.

\begin{definition}[Representable hypothesis class]
\label{def:representable_hypotheses}
A hypothesis class $\mathcal{H}$ is representable if there exists an injection $\textrm{enc}:\mathcal{H} \to \{0,1\}^\star$. Equivalently, $\mathcal{H}$ is at most countable. We implicitly identify hypotheses with their encodings whenever a learning procedure is required to output a hypothesis.
\end{definition}

The representability is deliberately minimal: it does not impose efficiency, runtime bounds, or any specific computational model. It only enforces the operational constraint that an output hypothesis must be nameable by a finite classical string. This is precisely the point at which purely set-theoretic pathologies can enter if left unchecked. When $\mathcal{H}$ contains objects that require genuinely non-finitary descriptions (for example, arbitrary subsets of a continuum), quantifying over ``all learners'' can silently quantify over non-operational outputs. By keeping $\mathcal{H}$ representable, PL ensures that any difficulty in learnability comes from the \emph{access model} (what information can be extracted) rather than from unnameable outputs.

The access model is captured by specifying, for each sample budget $d$, which conditional output laws are physically attainable.

\begin{definition}[Admissible protocol family]
\label{def:protocol_family}
Fix a learning task $(\Theta,\mathcal{H},U)$ with representable $\mathcal{H}$. An admissible protocol family is a sequence $\mathfrak{L}=\{\mathfrak{L}_d\}_{d\in\mathbb{N}}$ where each $\mathfrak{L}_d$ is a set of Markov kernels
\begin{eqnarray}
Q:\Theta \to \hat{\Delta}(\mathcal{H}), \quad \theta \mapsto Q(\cdot|\theta),
\label{eq:kernel_def}
\end{eqnarray}
interpreted as: with sample budget $d$, the learner can output a random hypothesis $H \sim Q(\cdot|\theta)$ when the environment is $\theta$. We assume $\mathfrak{L}_d$ is closed under classical post-processing: if $Q \in \mathfrak{L}_d$ and $\Pi:\mathcal{H} \to \hat{\Delta}(\mathcal{H})$ is a Markov kernel (a randomized relabeling), then $\Pi \circ Q \in \mathfrak{L}_d$. We also assume convexity: if $Q_1,Q_2 \in \mathfrak{L}_d$ and $\lambda \in [0,1]$, then $\lambda Q_1 + (1-\lambda) Q_2 \in \mathfrak{L}_d$.
\end{definition}

This kernel-based viewpoint is the operational heart of PL. A protocol may use internal randomness, adaptively interrogate its samples, or carry out complicated physical dynamics; nevertheless, once we only care about whether it succeeds, all of that structure collapses to the conditional output law $Q(\cdot|\theta)$. The family $\mathfrak{L}_d$ is therefore the correct object to encode physical constraints: it is the set of all behaviors that are achievable with $d$ units of access to the environment. The closure assumptions reflect basic operational facts: one may randomize between two procedures, and any physically produced classical record can be classically reprocessed without violating physical laws. Different physical models correspond to different choices of $\mathfrak{L}$ (classical i.i.d. access, $d$-copy quantum access, no-signaling access, coarse-grained access, and so on). With these primitives in place, we can now state the learnability notion that PL proposes.

\subsection{Physics-aware learnability}

We now define the learnability relative to an admissible protocol family. The definition is intentionally parallel to standard PAC/EMX definitions, with the crucial change that quantification over learners is restricted to $\mathfrak{L}$ rather than ranging over arbitrary set-theoretic functions.

\begin{definition}[PL learnability]
\label{def:pl_learnability}
Let $(\Theta, \mathcal{H}, U)$ be a learning task and let $\mathfrak{L}=\{\mathfrak{L}_d\}_{d \in \mathbb{N}}$ be an admissible protocol family. Fix $\epsilon,\delta \in (0,1)$. We say that the task is $(\epsilon,\delta)$-learnable in the PL sense relative to $\mathfrak{L}$ if there exists $d \in \mathbb{N}$ and $Q \in \mathfrak{L}_d$ such that for every $\theta \in \Theta$,
\begin{eqnarray}
\Pr_{H\sim Q(\cdot|\theta)}\left[ U(\theta,H) \ge \textrm{opt}_{\mathcal{H}}(\theta) - \epsilon \right] \ge 1-\delta.
\label{eq:pl_guarantee}
\end{eqnarray}
\end{definition}

Several points are worth drawing out explicitly, because they capture what is genuinely new compared to the standard set-theoretic treatment.

\smallskip
\paragraph*{\bf (i) From learners-as-functions to learners-as-physical behaviors.}
Classically, one often writes a learner as a function $G$ that maps samples to hypotheses, and then defines learnability by existential quantification over all such maps. PL replaces that existential quantification by the existence of a kernel $Q \in \mathfrak{L}_d$. This is more than a change of notation: a kernel is exactly the operational object that a physical learner induces. It packages all randomness, adaptivity, and internal degrees of freedom into a single conditional law that must be achievable under the physical access model.

\smallskip
\paragraph*{\bf (ii) Learnability becomes a statement about a \emph{pair}: task + access model.}
{\bf Definition~\ref{def:pl_learnability}} makes explicit that learnability is not an absolute statement about $(\Theta,\mathcal{H},U)$ alone, but a relative statement about the pair $((\Theta,\mathcal{H},U),\mathfrak{L})$. This is not a weakness. Rather, it is the formal expression of the empirical fact that data access is never free. A task might be learnable under classical i.i.d. sampling but not under one-shot quantum access; it might be unlearnable at infinite precision but become learnable once the laboratory interface is coarse-grained. PL is designed to articulate such distinctions without changing the statistical objective.

\smallskip
\paragraph*{\bf (iii) The ``physics'' enters only through admissibility, not through utility.}
The utility $U$ expresses what is valued (for EMX, captured mass; for state discrimination, success probability; for estimation, fidelity or risk), while $\mathfrak{L}$ expresses what is possible. This separation is methodologically important, because it lets one hold the learning objective fixed and ask how different physical constraints reshape the feasible set of behaviors. In particular, PL avoids the ambiguity of mixing physical constraints into the performance metric itself.

\smallskip
\paragraph*{\bf (iv) Why this matters for the Ben-David undecidability phenomenon.}
The independence result for EMX arises when learners are allowed to be arbitrary set-theoretic maps on samples from an uncountable domain~\cite{BenDavid2019}. PL intervenes exactly at the operational boundary. In other words, the output space is representable, and the admissible class $\mathfrak{L}$ is specified by a physical access model. When $\mathfrak{L}_d$ is defined by finitary physical constraints (for instance, ``POVMs on $d$ copies'' or ``functions of a discretized observation''), the resulting learnability statements no longer quantify over all abstract functions on a continuum. As we will see later, this turns the ZFC-independence example into a provably learnable operational task under finite precision.

\smallskip
\paragraph*{\bf (v) Sample complexity as a resource.}
Although {\bf Definition~\ref{def:pl_learnability}} is phrased as an existence statement, it naturally induces a notion of \emph{PL sample complexity} (or more generally, resource complexity): the smallest $d$ for which there exists an admissible kernel achieving Eq.~(\ref{eq:pl_guarantee}). In classical settings, $d$ is the number of i.i.d.\ samples. On the other hand, in quantum settings, it becomes a copy complexity; under coarse-graining it can be interpreted as the number of interface queries. In all cases, $d$ is the parameter through which operational constraints enter quantitatively.


\subsection{Recovering standard learning as a degenerate physical model}

The PL template subsumes the classical i.i.d. sample model as a special case. The purpose of the next result is not to re-prove the conventional learning theory, but to certify that PL does not alter the semantics of standard models when the admissible class is chosen to match them.

We illustrate this in the EMX setting. Let $\mathcal{X}$ be a domain and $\mathcal{F} \subseteq \{0,1\}^{\mathcal{X}}$ be a representable family, identified with its range of encodings. Let $\Theta$ be a class of finitely supported distributions $P$ over $\mathcal{X}$. The EMX utility is
\begin{eqnarray}
U(P,f) := P(f),
\label{eq:emx_utility}
\end{eqnarray}
and $\textrm{opt}_{\mathcal{F}}(P)=\sup_{f\in\mathcal{F}}P(f)$ as in Eq.~(\ref{eq:opt_emx}).

Define the classical i.i.d. admissible family $\mathfrak{L}^{\textrm{cl}}$ by declaring that $Q \in \mathfrak{L}^{\textrm{cl}}_d$ if and only if there exists a (possibly randomized) proper learning rule that, on input a sample $S=(x_1,\ldots,x_d) \sim P^d$, outputs $f \in \mathcal{F}$ with law $Q(\cdot | P)$.

\begin{theorem}[PL-EMX coincides with standard EMX under classical i.i.d. access]
\label{thm:pl_equals_emx_classical}
A family $\mathcal{F}$ is $(\epsilon,\delta)$-EMX learnable over $\Theta$ in the standard sense if and only if the learning task $(\Theta,\mathcal{F},U)$ with $U$ given by Eq.~(\ref{eq:emx_utility}) is $(\epsilon,\delta)$-learnable in the PL sense relative to $\mathfrak{L}^{\textrm{cl}}$.
\end{theorem}

\begin{proof}
We unpack the two notions and observe that they coincide once the admissible family is chosen to be the classical i.i.d. model.

($\Rightarrow$) Suppose $\mathcal{F}$ is standard $(\epsilon,\delta)$-EMX learnable over $\Theta$. Then, there exists a sample size $d$ and a proper (possibly randomized) learner $G$ such that for every $P \in \Theta$,
\begin{eqnarray}
\Pr_{S \sim P^d}\left[P\big(G(S)\big) \ge \textrm{opt}_{\mathcal{F}}(P)-\epsilon\right] \ge 1-\delta.
\label{eq:standard_emx}
\end{eqnarray}
Let $Q(\cdot | P)$ denote the output distribution of $G(S)$ when $S \sim P^d$ (if $G$ is randomized, the probability in Eq.~(\ref{eq:standard_emx}) also averages over its internal randomness). By construction, this conditional law is realizable under classical sampling, hence $Q\in\mathfrak{L}^{\textrm{cl}}_d$. Moreover, the event inside Eq.~(\ref{eq:standard_emx}) depends only on the output hypothesis, so Eq.~(\ref{eq:standard_emx}) is exactly the PL condition in Eq.~(\ref{eq:pl_guarantee}) for the task $(\Theta,\mathcal{F},U)$.

($\Leftarrow$) Conversely, suppose the PL condition holds relative to $\mathfrak{L}^{\textrm{cl}}$. Then, there exist $d$ and a kernel $Q \in \mathfrak{L}^{\textrm{cl}}_d$ such that for every $P \in \Theta$,
\begin{eqnarray}
\Pr_{H\sim Q(\cdot|P)}\left[P(H) \ge \textrm{opt}_{\mathcal{F}}(P)-\epsilon\right] \ge 1-\delta.
\end{eqnarray}
By definition of $\mathfrak{L}^{\textrm{cl}}_d$, the kernel $Q(\cdot | P)$ is implemented by some proper (possibly randomized) learning rule on i.i.d. samples of size $d$. Running that rule defines a standard EMX learner whose output law is $Q$, and the displayed inequality is exactly the standard EMX guarantee.
\end{proof}

{\bf Theorem~\ref{thm:pl_equals_emx_classical}} is a semantic checkpoint. PL does not compete with classical learning theory; it parameterizes it by making the access model explicit. When $\mathfrak{L}$ is chosen to match the classical i.i.d. interface, nothing changes. The benefit is that once the access model is explicit, we can vary it in principled ways. We next do so for two canonical physical restrictions: quantum access (where admissibility is dictated by measurement theory) and finite precision (where admissibility is dictated by coarse-graining).

\subsection{Quantum PL: admissible learners as POVMs on $d$ copies}

Quantum mechanics provides a particularly clean instantiation of the PL idea because the access model is sharply constrained by postulates of the theory. This style of thinking---treating the interface to quantum data as part of the learning model---is also standard in quantum machine learning~\cite{Biamonte2017QML,Ciliberto2018QML}. If the environment is an unknown quantum state, then ``observing a sample'' means receiving a physical system in that state, and the only admissible interactions are quantum operations followed by a classical readout. The sample budget $d$ is not an accounting device but a physical resource: no-cloning prevents the learner from generating additional i.i.d. samples from a single specimen (cf. {\bf Theorem~\ref{thm:no_cloning}}).

Let $\mathcal{H}$ be a (finite or countable) set of hypotheses. Let $\mathcal{K}$ be a finite-dimensional Hilbert space. An environment is a density operator $\hat{\rho}$ on $\mathcal{K}$, and the $d$-sample resource is $\hat{\rho}^{\otimes d}$ on $\mathcal{K}^{\otimes d}$. A general quantum-to-classical protocol that outputs $h \in \mathcal{H}$ can be modeled as a quantum measurement with outcomes labeled by $\mathcal{H}$. The next theorem states that this modeling is exact: the admissible family $\mathfrak{L}^{\rm q}_d$ for $d$-copy quantum access is precisely the set of kernels induced by POVMs on $\mathcal{K}^{\otimes d}$ (together with the classical post-processing closure already built into {\bf Definition~\ref{def:protocol_family})}.

\begin{figure}[t]
\centering
\includegraphics[width=0.95\linewidth]{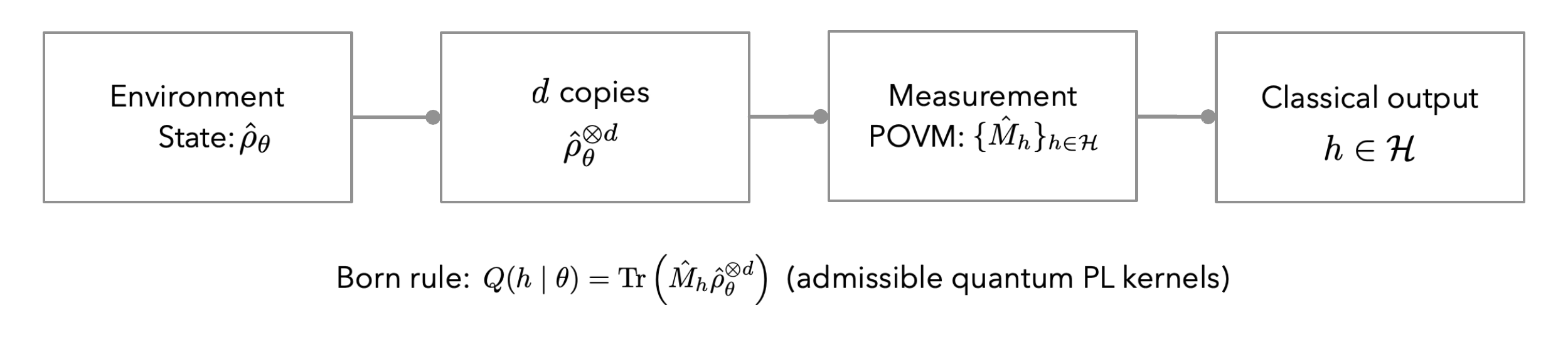}
\caption{Quantum instantiation of PL. Under $d$-copy access to an unknown state $\hat{\rho}_\theta$, any admissible protocol reduces to a POVM $\{\hat{M}_h\}_{h \in \mathcal{H}}$ on $\hat{\rho}_\theta^{\otimes d}$ followed by classical post-processing ({\bf Theorem~\ref{thm:povm_representation}}). The copy budget $d$ is a physical resource due to the no-cloning theorem ({\bf Theorem~\ref{thm:no_cloning}}).}
\label{fig:quantum_pl}
\end{figure}

\begin{theorem}[POVM representation of quantum PL protocols~\cite{NielsenChuang2010,Watrous2018}]
\label{thm:povm_representation}
Fix $d \in \mathbb{N}$ and a finite or countable hypothesis set $\mathcal{H}$. A Markov kernel $Q(\cdot | \hat{\rho})$ from density matrices on $\mathcal{K}$ to $\hat{\Delta}(\mathcal{H})$ is physically realizable by a quantum protocol acting on $\hat{\rho}^{\otimes d}$ and producing a classical outcome in $\mathcal{H}$ if and only if there exists a POVM $\{\hat{M}_h\}_{h \in \mathcal{H}}$ on $\mathcal{K}^{\otimes d}$ such that
\begin{eqnarray}
Q(h | \hat{\rho}) = \tr{\hat{M}_h \hat{\rho}^{\otimes d}} \quad (\forall\, h \in \mathcal{H}).
\label{eq:qkernel_povm}
\end{eqnarray}
\end{theorem}

\begin{proof}
($\Rightarrow$) Consider any physically realizable protocol on $\hat{\rho}^{\otimes d}$ with classical output in $\mathcal{H}$. Without loss of generality, model the protocol as a completely positive trace-preserving (CPTP) map $\hat{\Phi}$ from operators on $\mathcal{K}^{\otimes d}$ to operators on a classical register $\mathcal{C}$ spanned by $\{\ket{h}\}_{h \in \mathcal{H}}$, such that $\hat{\Phi}(\hat{\rho}^{\otimes d})$ is diagonal in this basis and its diagonal entries give the output distribution. Then,
\begin{eqnarray}
Q(h | \hat{\rho}) = \tr{ \ketbra{h}{h} \hat{\Phi}(\hat{\rho}^{\otimes d})}.
\label{eq:classical_register_probability}
\end{eqnarray}
Let $\hat{\Phi}^\dagger$ denote the Hilbert--Schmidt adjoint of $\hat{\Phi}$, defined by $\tr{\hat{A}\hat{\Phi}(\hat{B})}=\tr{\hat{\Phi}^\dagger(\hat{A})\hat{B}}$ for all operators $\hat{A}$ on $\mathcal{C}$ and $\hat{B}$ on $\mathcal{K}^{\otimes d}$. Define
\begin{eqnarray}
\hat{M}_h := \hat{\Phi}^\dagger\left(\ketbra{h}{h}\right).
\label{eq:Mh_def}
\end{eqnarray}
Since $\hat{\Phi}$ is completely positive, $\hat{\Phi}^\dagger$ is completely positive, hence each $\hat{M}_h$ is positive semidefinite. Moreover, since $\hat{\Phi}$ is trace-preserving, $\hat{\Phi}^\dagger$ is unital, so
\begin{eqnarray}
\sum_{h \in \mathcal{H}} \hat{M}_h = \hat{\Phi}^\dagger\left(\sum_{h \in \mathcal{H}} \ketbra{h}{h}\right) = \hat{\Phi}^\dagger(I_{\mathcal{C}}) = I_{\mathcal{K}^{\otimes d}},
\label{eq:povm_sum_to_identity}
\end{eqnarray}
which shows $\{\hat{M}_h\}$ is a POVM. Finally, combining Eq.~(\ref{eq:classical_register_probability}) with adjointness yields
\begin{eqnarray}
Q(h | \hat{\rho}) = \tr{\ketbra{h}{h}\hat{\Phi}(\hat{\rho}^{\otimes d})} = \tr{\hat{\Phi}^\dagger(\ketbra{h}{h})\hat{\rho}^{\otimes d}} = \tr{\hat{M}_h \hat{\rho}^{\otimes d}},
\label{eq:povm_probability}
\end{eqnarray}
as required.

($\Leftarrow$) Conversely, given a POVM $\{\hat{M}_h\}$, define the protocol that performs this measurement on $\hat{\rho}^{\otimes d}$ and outputs the classical label $h$. The Born rule gives exactly Eq.~(\ref{eq:qkernel_povm}).
\end{proof}

{\bf Theorem~\ref{thm:povm_representation}} is the quantum analogue of the classical statement ``a learner is a function of the sample.'' This states that a quantum PL learner is, essentially, a choice of measurement on $d$ copies and a relabeling of outcomes. In this form, the physics enters with complete clarity: admissibility is encoded in the requirement that output laws arise from POVMs on $\hat{\rho}^{\otimes d}$, and the resource parameter $d$ becomes a copy complexity.

It is also useful to see explicitly how this fits into the $(\Theta,\mathcal{H},U)$ template. In quantum PL, the environment parameter $\theta$ selects a state $\hat{\rho}_\theta$ (or more generally a preparation procedure); the hypothesis space $\mathcal{H}$ is the finite or countable set of classical labels the learner can output; and the utility $U(\theta,h)$ can encode any operational performance criterion (e.g., correct identification probability, fidelity thresholds, or task-specific payoffs). The admissible family $\mathfrak{L}^{\rm q}_d$ is then exactly ``POVMs on $\hat{\rho}_\theta^{\otimes d}$ + classical post-processing.'' This example matters for our broader thesis because it shows that $\mathfrak{L}$ is not an ad hoc restriction: in important physical regimes it can be derived directly from standard axioms, and it yields a mathematically crisp characterization of admissible inference.

\subsection{Coarse-graining and finite precision: a PL reduction principle}

Quantum constraints are one way that physics restricts admissible inference, but they are not the only way. Even in purely classical laboratories, a universal restriction is finite precision. When an underlying variable takes values in a continuum, the measurement interface does not return an exact real number; it returns a digitized label (a bin index, a finite-resolution pixel, a discretized readout, etc.). In PL terms, this means the learner does not access $x \in \mathcal{X}$ itself, but only a coarse-grained observation $\pi(x)$ through a specified interface map $\pi$.

Let $\mathcal{X}$ be a domain and let $\pi:\mathcal{X} \to \mathcal{Y}$ be a measurable map into a countable set $\mathcal{Y}$. Think of $\pi$ as an $\ell$-bit digitization device. It reports a discretized label $\pi(x)$ rather than the full continuum-valued $x$. For a distribution $P$ on $\mathcal{X}$, let $\pi_\# P$ denote the pushforward distribution on $\mathcal{Y}$,
\begin{eqnarray}
(\pi_\# P)(B) := P\left(\pi^{-1}(B)\right) \quad (\forall \, B \subseteq \mathcal{Y}).
\label{eq:pushforward_def}
\end{eqnarray}
Given a family $\mathcal{G} \subseteq \{0,1\}^{\mathcal{Y}}$, define its pullback to $\mathcal{X}$ by
\begin{eqnarray}
\pi^{-1}(\mathcal{G}) := \left\{ g \circ \pi: g \in \mathcal{G} \right\} \subseteq \{0,1\}^{\mathcal{X}}.
\label{eq:pullback_class}
\end{eqnarray}

\begin{figure}[t]
\centering
\includegraphics[width=0.70\linewidth]{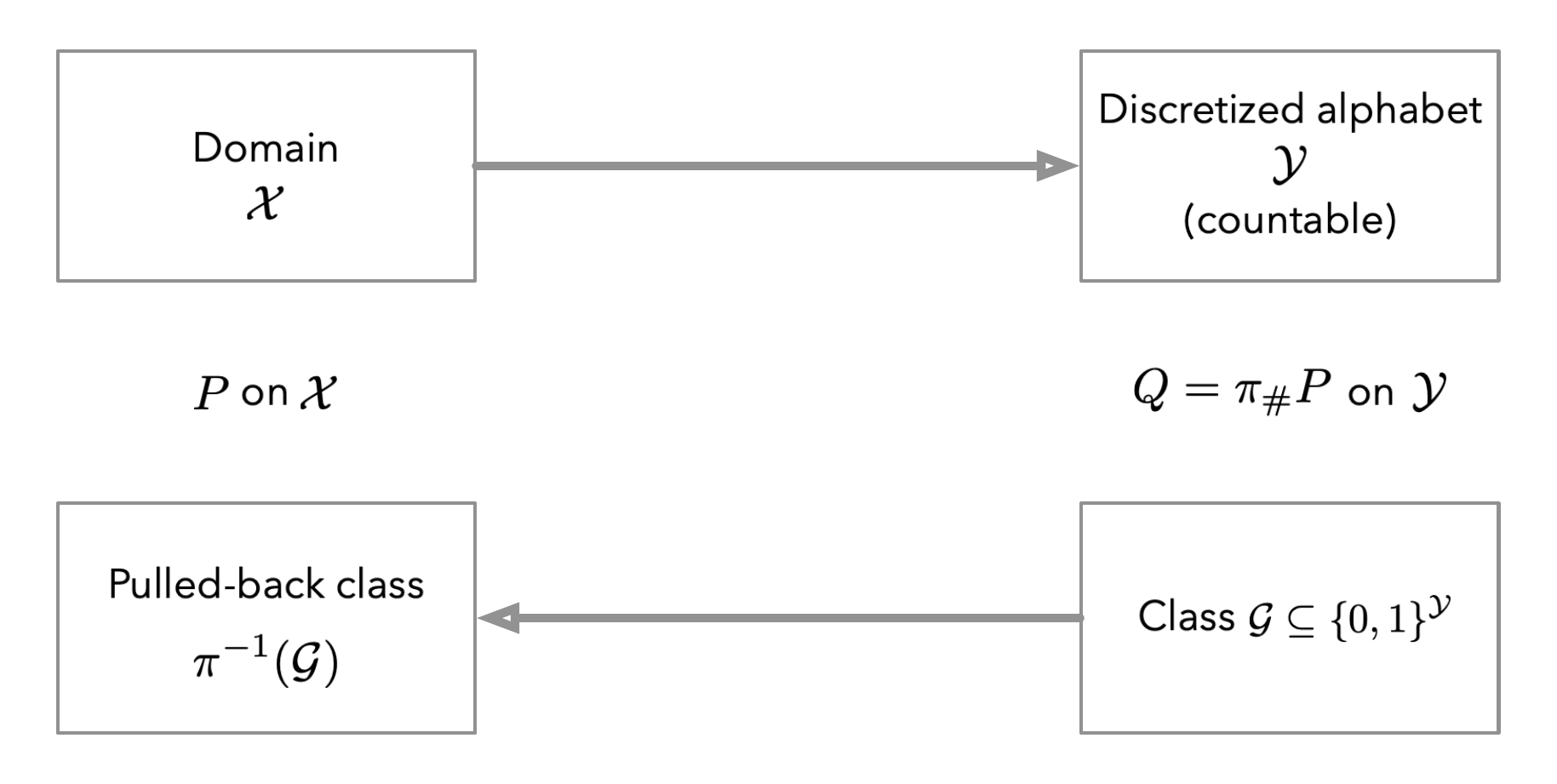}
\caption{Coarse-graining reduction. A finite-precision interface $\pi:\mathcal{X} \to \mathcal{Y}$ induces a pushforward distribution $Q=\pi_\#P$ on the countable alphabet and a pulled-back hypothesis class $\pi^{-1}(\mathcal{G})$ on $\mathcal{X}$. The identity $P(g\circ\pi)=Q(g)$ (Eq.~\ref{eq:utility_preserved}) implies that learning on $\mathcal{Y}$ yields PL learning on $\mathcal{X}$ under coarse-grained access (Theorem~\ref{thm:coarse_graining_reduction}).}
\label{fig:coarse_graining}
\end{figure}

\begin{theorem}[Coarse-graining reduction for EMX]
\label{thm:coarse_graining_reduction}
Fix $\epsilon,\delta \in (0,1)$. Let $\mathcal{G} \subseteq \{0,1\}^{\mathcal{Y}}$ be a representable family over a countable $\mathcal{Y}$, and let $\Theta$ be any class of finitely supported distributions over $\mathcal{X}$. Consider the EMX task on $\mathcal{X}$ with hypothesis family $\pi^{-1}(\mathcal{G})$ and utility $U(P,f)=P(f)$. Suppose that $\mathcal{G}$ is $(\epsilon,\delta)$-EMX learnable over the pushed-forward class $\pi_\# \Theta:=\{\pi_\# P: P \in \Theta\}$. Then $\pi^{-1}(\mathcal{G})$ is $(\epsilon,\delta)$-learnable in the PL sense relative to the access model that reveals only the discretized samples $\pi(x)$.
\end{theorem}

\begin{proof}
Let $G$ be an $(\epsilon,\delta)$-EMX learner for $\mathcal{G}$ over $\pi_\#\Theta$ with sample size $d=d(\epsilon,\delta)$. We construct an admissible PL learner for the pulled-back problem on $\mathcal{X}$ by using $G$ as a subroutine on the \emph{observations} $\pi(x)$.

\smallskip
\paragraph*{\bf Construction of the PL protocol.} Under the coarse-grained access model, a $d$-sample from $P$ appears to the learner only through the discretized sequence $(\pi(x_1),\ldots,\pi(x_d))\in\mathcal{Y}^d$. The protocol runs $G$ on this sequence, obtains some $g\in\mathcal{G}$, and outputs the hypothesis
\begin{eqnarray}
f := g \circ \pi \in \pi^{-1}(\mathcal{G}).
\end{eqnarray}
This protocol is admissible by construction (it depends only on the available coarse-grained record).

\smallskip
\paragraph*{\bf Preservation of utilities.} Fix any $P\in\Theta$ and let $Q:=\pi_\# P\in\pi_\#\Theta$ be the induced distribution on $\mathcal{Y}$. For every $g\in\mathcal{G}$ we have, by definition of the pushforward measure,
\begin{eqnarray}
P(g \circ \pi) = P\left(\pi^{-1}\big(\{y:g(y)=1\}\big)\right) = Q(g).
\label{eq:utility_preserved}
\end{eqnarray}
This identity says that the probability mass captured by $f=g\circ\pi$ under $P$ is exactly the mass captured by $g$ under the induced distribution $Q$.

\smallskip
\paragraph*{\bf Preservation of optima.} Taking suprema over $g \in \mathcal{G}$ in Eq.~(\ref{eq:utility_preserved}) yields
\begin{eqnarray}
\textrm{opt}_{\pi^{-1}(\mathcal{G})}(P) = \sup_{g\in\mathcal{G}} P(g\circ\pi) = \sup_{g\in\mathcal{G}} Q(g) = \textrm{opt}_{\mathcal{G}}(Q).
\label{eq:opt_preserved}
\end{eqnarray}

\smallskip
\paragraph*{\bf Transfer of the EMX guarantee.} Since $G$ is an $(\epsilon,\delta)$-EMX learner for $\mathcal{G}$ under $Q$, with probability at least $1-\delta$ it outputs $g$ satisfying
\begin{eqnarray}
Q(g) \ge \textrm{opt}_{\mathcal{G}}(Q)-\epsilon.
\end{eqnarray}
Combining this with Eq.~(\ref{eq:utility_preserved}) and Eq.~(\ref{eq:opt_preserved}), the corresponding $f=g\circ\pi$ satisfies
\begin{eqnarray}
P(f)\ \ge\ \textrm{opt}_{\pi^{-1}(\mathcal{G})}(P)-\epsilon,
\label{eq:coarse_graining_guarantee}
\end{eqnarray}
which is exactly the PL guarantee Eq.~(\ref{eq:pl_guarantee}) for the coarse-grained access model.
\end{proof}

{\bf Theorem~\ref{thm:coarse_graining_reduction}} is not merely a technical lemma; it is a reduction principle that crystallizes the operational content of finite precision. Once the laboratory interface is a coarse-graining $\pi$, the learner can only distinguish environments through the induced distribution $Q=\pi_\#P$ on the countable alphabet $\mathcal{Y}$. In that sense, the physically meaningful learnability question is the learnability of the induced problem on $\mathcal{Y}$, not on the full continuum $\mathcal{X}$. The theorem makes this intuition exact: it shows that EMX learning under coarse-grained access is equivalent to standard EMX learning on the pushed-forward problem, with no loss in $(\epsilon,\delta)$ and with the same sample size.

This is precisely the kind of statement that PL is designed to enable. In purely set-theoretic formulations, it is easy to conflate the ``true'' domain $\mathcal{X}$ with what is operationally available. {\bf Theorem~\ref{thm:coarse_graining_reduction}} separates these: it tells us that if two environments agree after passing through the interface $\pi$, then no coarse-grained protocol can distinguish them, and the only hypotheses that can be meaningfully evaluated are those that depend on $\pi(x)$. As a consequence, learnability phenomena that are sensitive to the full set-theoretic structure of an uncountable $\mathcal{X}$---including the Ben-David undecidability example---need not survive under physically realistic interfaces. The next subsection makes this point concrete.

\subsection{A PL resolution of the EMX independence example under finite precision}

We now connect the coarse-graining principle to the canonical EMX class that exhibits ZFC-independence in the unconstrained model~\cite{BenDavid2019}. For a set $\mathcal{Y}$, denote by
\begin{eqnarray}
\mathcal{F}_{\textrm{fin}}^{\mathcal{Y}} := \big\{F \subseteq \mathcal{Y}: \mbox{$F$ is finite}\big\},
\label{eq:F_fin}
\end{eqnarray}
the family of all finite subsets (identified with their indicator functions). For a coarse-graining $\pi:\mathcal{X} \to \mathcal{Y}$, define the corresponding physically realizable finite-subset class on $\mathcal{X}$ by
\begin{eqnarray}
\mathcal{F}_{\textrm{fin}}^{(\pi)} := \pi^{-1}\!\left(\mathcal{F}_{\textrm{fin}}^{\mathcal{Y}}\right) = \big\{\pi^{-1}(F): F \in \mathcal{F}_{\textrm{fin}}^{\mathcal{Y}}\big\}.
\label{eq:F_fin_pi}
\end{eqnarray}
This class should be read operationally: a learner that only observes $\pi(x)$ cannot output an arbitrary finite subset of $\mathcal{X}$, but it can output a finite collection of discretization cells, i.e. a set of the form $\pi^{-1}(F)$.

\begin{corollary}[Finite-precision EMX is provably learnable]
\label{cor:finite_precision_emx}
Let $\mathcal{Y}$ be countable and let $\pi:\mathcal{X} \to \mathcal{Y}$ be any map. Then, $\mathcal{F}_{\textrm{fin}}^{(\pi)}$ is $(1/3, 1/3)$-learnable in the PL sense under the access model that reveals only $\pi(x)$.
\end{corollary}

\begin{proof}
We apply {\bf Theorem~\ref{thm:coarse_graining_reduction}} with $\mathcal{G}=\mathcal{F}_{\textrm{fin}}^{\mathcal{Y}}$. Because $\mathcal{Y}$ is countable, the weak EMX learnability of $\mathcal{F}_{\textrm{fin}}^{\mathcal{Y}}$ over finitely supported distributions is provable in ZFC (cf. {\bf Corollary~\ref{cor:countable_emx}} from Sec.~III). The reduction theorem then transfers this learnability statement back to $\mathcal{X}$ under the coarse-grained access model, yielding the claim for $\mathcal{F}_{\textrm{fin}}^{(\pi)}$.
\end{proof}

{\bf Corollary~\ref{cor:finite_precision_emx}} encapsulates the operational moral of the undecidability phenomenon. In the unconstrained set-theoretic model, learnability of the class of all finite subsets of $[0,1]$ is intertwined with cardinal arithmetic and becomes independent of ZFC~\cite{BenDavid2019}. In PL, the interface to $[0,1]$ is not the identity map but a physically chosen coarse-graining $\pi$ that yields a countable observation alphabet and restricts the learner to hypotheses that can be named at that resolution. Under this operational restriction, the corresponding EMX task becomes provably learnable.

This does not claim that physics ``decides'' set theory, nor that physical principles settle questions such as the continuum hypothesis. What it does claim is more modest and more relevant to learning: the objects that drive set-theoretic independence in the EMX example are non-operational idealizations. The independence arises when one insists on (i) an identity interface to an uncountable domain---as if one could observe exact real numbers---and (ii) quantification over arbitrary set-theoretic learners and hypothesis sets on that domain. PL replaces these idealizations by the operational content that is actually present in the laboratory: a finite-resolution interface $\pi$ and a representable output space.

Once $\pi$ is specified, the learning task itself changes in a controlled and principled way: the learner is no longer trying to name an arbitrary finite subset of $\mathcal{X}$, but rather a finite collection of distinguishable regions $\pi^{-1}(F)$. This is exactly what an experiment can implement, because the only evidence available is the stream of coarse-grained labels $\pi(x_1),\ldots,\pi(x_d)$. At that point, the problem collapses to an EMX instance on a countable alphabet $\mathcal{Y}$, and countability brings us back into the realm where ZFC provides determinate, finitary proofs of learnability (Sec.~III). In other words, the apparent ``mystery'' is not that physics resolves set theory, but that the set-theoretic independence was tied to an unphysical limiting idealization.

Finally, this perspective also clarifies how one should interpret the classical undecidability result in practice. One may take finer and finer coarse-grainings $\pi$ and ask how the PL sample complexity behaves as the resolution increases. In the limit of infinite precision (which is not physically attainable), the induced discrete alphabets may approximate the full continuum and the set-theoretic subtleties can re-emerge. PL therefore does not erase foundational issues; rather, it localizes them: it tells us that any physically realizable learning problem is posed at finite resolution and hence admits a well-defined operational learnability status, while undecidability is a warning about what can happen when one conflates that operational problem with a non-operational mathematical limit.

\section{Case Studies and Technical Program}

The PL formalism of the previous section is deliberately abstract: it separates the definition of the learning objective from the specification of admissible physical protocols. The purpose of this section is to show that this separation has concrete mathematical consequences. We focus on three case studies that together illustrate the main methodological claim of this paper. First, we revisit the canonical EMX class whose learnability is independent of ZFC and show that, once one replaces ``access to a continuum'' by a finite-precision interface, the corresponding operational learning problem admits an explicit $(\epsilon,\delta)$ learner with sharp sample complexity. Second, we use the quantum instantiation of PL to derive a copy-complexity lower bound for a basic identification task; this exhibits, in a particularly transparent form, how no-cloning turns sample size into a physical resource. Third, we identify a broad regime in which the PL feasibility question is algorithmically decidable from a finite description of the physical constraints, and we explain why no-signaling constraints fall into this regime.

\subsection{Finite-subset EMX under an operational ordering}

The ZFC-independence result of Ben-David \emph{et al.} concerns the family $\mathcal{F}_{\textrm{fin}}^{[0,1]}$ of all finite subsets of $[0,1]$ and the class of finitely supported distributions on $[0,1]$~\cite{BenDavid2019}. The operational perspective suggests that the relevant question is not learnability over the bare continuum, but the learnability relative to an interface that assigns to each observation a finite classical description. In Sec.~IV, we formalized this via the coarse-graining maps $\pi:\mathcal{X} \to \mathcal{Y}$ with countable range. Here we complement that structural reduction with an explicit learner and a simple, quantitative analysis.

Let $\mathcal{X}$ be countable and fix an injection $\idx:\mathcal{X}\to\mathbb{N}$, viewed as a physically available naming scheme. For $t \in \mathbb{N}$, let us define the initial segment
\begin{eqnarray}
A_t := \{x \in \mathcal{X}: \idx(x) \le t\}.
\label{eq:initial_segment_At}
\end{eqnarray}
Since $\idx$ is injective, $A_t$ is finite for every $t$. Consider the EMX family $\mathcal{F}_{\textrm{fin}}^{\mathcal{X}}$ and restrict distributions to be finitely supported, so that $\textrm{opt}_{\mathcal{F}_{\textrm{fin}}^{\mathcal{X}}}(P)=1$.

\begin{theorem}[A quantile learner for $\mathcal{F}_{\textrm{fin}}^{\mathcal{X}}$]
\label{thm:quantile_emx_countable}
Let $\mathcal{X}$ be countable and let $\idx:\mathcal{X}\to\mathbb{N}$ be an injection. Define a proper learner $G$ as follows: given a sample $S=(x_1,\ldots,x_d) \in \mathcal{X}^d$, let
\begin{eqnarray}
T(S) := \max_{1\le j \le d}\ \idx(x_j),
\quad
G(S) := A_{T(S)} \in \mathcal{F}_{\textrm{fin}}^{\mathcal{X}}.
\label{eq:quantile_learner_def}
\end{eqnarray}
Then, for every $\epsilon,\delta\in(0,1)$ and every finitely supported distribution $P$ over $\mathcal{X}$,
\begin{eqnarray}
\Pr_{S\sim P^d}\left[P\big(G(S)\big) \ge 1-\epsilon\right] \ge 1-\delta
\label{eq:quantile_emx_guarantee}
\end{eqnarray}
whenever
\begin{eqnarray}
d \ge \frac{\ln(1/\delta)}{-\ln(1-\epsilon)}.
\label{eq:quantile_sample_complexity}
\end{eqnarray}
In particular, $d=O\left(\frac{1}{\epsilon}\ln\frac{1}{\delta}\right)$ for $\epsilon$ bounded away from $1$.
\end{theorem}

\begin{proof}
Fix a finitely supported distribution $P$ on $\mathcal{X}$. Since $P$ has finite support, $\lim_{t \to \infty}P(A_t)=1$. Let $t_\epsilon$ be the smallest integer such that
\begin{eqnarray}
P(A_{t_\epsilon}) \ge 1-\epsilon.
\label{eq:t_eps_def}
\end{eqnarray}
By minimality, $P(A_{t_\epsilon-1}) < 1-\epsilon$. Consider the event
\begin{eqnarray}
E := \big\{\,T(S) < t_\epsilon \big\}.
\label{eq:event_E}
\end{eqnarray}
On $E$, all sample points satisfy $\idx(x_j)\le t_\epsilon-1$, hence $x_j\in A_{t_\epsilon-1}$ for every $j$. Therefore,
\begin{eqnarray}
\Pr_{S\sim P^d}[E] = P(A_{t_\epsilon-1})^d < (1-\epsilon)^d.
\label{eq:E_probability}
\end{eqnarray}
On the complement event $E^c$, we have $T(S) \ge t_\epsilon$, hence $A_{t_\epsilon} \subseteq A_{T(S)}=G(S)$. Thus,
\begin{eqnarray}
P\big(G(S)\big) \ge P(A_{t_\epsilon}) \ge 1-\epsilon.
\label{eq:good_event_implies_mass}
\end{eqnarray}
By combining Eq.~(\ref{eq:E_probability}) and Eq.~(\ref{eq:good_event_implies_mass}), we can yield
\begin{eqnarray}
\Pr_{S\sim P^d}\left[P\big(G(S)\big) \ge 1-\epsilon\right] \ge 1-(1-\epsilon)^d.
\label{eq:final_bound}
\end{eqnarray}
The choice of Eq.~(\ref{eq:quantile_sample_complexity}) ensures $(1-\epsilon)^d\le\delta$, proving Eq.~(\ref{eq:quantile_emx_guarantee}).
\end{proof}

{\bf Theorem~\ref{thm:quantile_emx_countable}} is deliberately elementary: it is a one-line algorithm whose analysis uses only the monotonicity of initial segments and the i.i.d. assumption. Yet, it captures the operational heart of the EMX task for finite-support distributions. In other words, the objective is to recover a high-probability portion of the support, and a single ``quantile-like'' statistic of the sample suffices once the domain has been endowed with a physically meaningful naming scheme.

Via coarse-graining, this yields an explicit PL learner for the finite-precision variant of the Ben-David class.

\begin{corollary}[Explicit PL learner under finite precision]
\label{cor:quantile_emx_coarse_grain}
Let $\pi:\mathcal{X} \to \mathcal{Y}$ be a coarse-graining with countable $\mathcal{Y}$, and fix an injection $\idx:\mathcal{Y}\to\mathbb{N}$. Consider the hypothesis class $\mathcal{F}_{\textrm{fin}}^{(\pi)}=\pi^{-1}(\mathcal{F}_{\textrm{fin}}^{\mathcal{Y}})$ from Eq.~(\ref{eq:F_fin_pi}) and the access model that reveals only $\pi(x)$. Then, for every $\epsilon,\delta \in (0,1)$, the PL task is $(\epsilon,\delta)$-learnable with sample size $d$ as in Eq.~(\ref{eq:quantile_sample_complexity}). An explicit learner is obtained by applying Eq.~(\ref{eq:quantile_learner_def}) to the discretized samples and outputting the corresponding union of discretization cells.
\end{corollary}

\begin{proof}
Let $P$ be a finitely supported distribution on $\mathcal{X}$ and let $Q=\pi_\#P$ be its pushforward to $\mathcal{Y}$. Since $\textrm{supp}(P)$ is finite, so is $\textrm{supp}(Q)$, and hence $\textrm{opt}_{\mathcal{F}_{\textrm{fin}}^{\mathcal{Y}}}(Q)=1$. Apply {\bf Theorem~\ref{thm:quantile_emx_countable}} on $(\mathcal{Y},Q)$ to obtain, with probability at least $1-\delta$, a finite set $F \subseteq \mathcal{Y}$ such that $Q(F) \ge 1-\epsilon$. Output $\pi^{-1}(F) \in \mathcal{F}_{\textrm{fin}}^{(\pi)}$. By definition of pushforward,
\begin{eqnarray}
P\left(\pi^{-1}(F)\right) = Q(F) \ge 1-\epsilon.
\label{eq:pullback_probability}
\end{eqnarray}
This is exactly the PL guarantee for this access model.
\end{proof}

{\bf Corollary~\ref{cor:quantile_emx_coarse_grain}} strengthens the qualitative statement of {\bf Corollary~\ref{cor:finite_precision_emx}} by providing an explicit $(\epsilon,\delta)$ sample complexity and a concrete protocol. In particular, the operational variant of the EMX problem that arises under finite precision is not merely ``provably learnable in ZFC'' but it is learnable by an algorithm whose behavior has an immediate physical interpretation. Thus, it returns the union of all discretization cells whose indices do not exceed the maximum observed index. In this sense, the ZFC-independence of $\mathcal{F}_{\textrm{fin}}^{[0,1]}$ is revealed to be highly sensitive to the idealization that the learner has access to the full continuum and may output arbitrary finite subsets thereof, rather than the finite information actually provided by a measurement interface.

\subsection{Quantum case study: copy complexity from the geometry of states}

In quantum PL, the admissible learners are constrained not only by finite description, but by the structure of quantum measurements on finitely many copies of an unknown state ({\bf Theorem~\ref{thm:povm_representation}}). A basic and physically compelling learning task in this setting is the state identification~\cite{Huang2020Predicting,Zhao2024Learning}: the environment is promised to be one of two possible states, and the learner must output which one. Although this task is conceptually simpler than EMX, it plays a similar foundational role. It isolates, in its cleanest form, how physical laws constrain inference.

Let $\hat{\rho}_0$ and $\hat{\rho}_1$ be density operators on a finite-dimensional Hilbert space $\mathcal{K}$. Consider the binary hypothesis task with environment set $\Theta=\{0,1\}$, hypothesis set $\mathcal{H}=\{0,1\}$, and utility
\begin{eqnarray}
U(\theta,h) := \mathbb{1}\{\theta=h\}.
\label{eq:binary_utility}
\end{eqnarray}
Under $d$-copy quantum access, an admissible protocol is a POVM $\{\hat{M}_0,\hat{M}_1\}$ on $\mathcal{K}^{\otimes d}$, and the induced decision rule is $Q(h | \theta)=\tr{\hat{M}_h \hat{\rho}_\theta^{\otimes d}}$. The following theorem gives the fundamental limitation on simultaneous success for both environments.

\begin{theorem}[A worst-case Helstrom bound~\cite{Helstrom1976,Watrous2018}]
\label{thm:helstrom_worst_case}
Let $\hat{\rho}_0,\hat{\rho}_1$ be density operators on $\mathcal{K}$ and let $d\in\mathbb{N}$. For any POVM $\{\hat{M}_0,\hat{M}_1\}$ on $\mathcal{K}^{\otimes d}$,
\begin{eqnarray}
\tr{\hat{M}_0\,\hat{\rho}_0^{\otimes d}} + \tr{\hat{M}_1 \hat{\rho}_1^{\otimes d}} \le 1 + \frac{1}{2}\left\|\hat{\rho}_0^{\otimes d}-\hat{\rho}_1^{\otimes d}\right\|_1.
\label{eq:helstrom_sum_bound}
\end{eqnarray}
Moreover, equality is achieved by choosing $\hat{M}_0$ to be the projector onto the support of the positive part of $\hat{\rho}_0^{\otimes d}-\hat{\rho}_1^{\otimes d}$ and $\hat{M}_1=\hat{\openone} - \hat{M}_0$.
\end{theorem}

\begin{proof}
Let $\hat{\Delta}:=\hat{\rho}_0^{\otimes d}-\hat{\rho}_1^{\otimes d}$, which is Hermitian and satisfies $\tr{\hat{\Delta}}=0$ because both states have trace $1$. Using $\hat{M}_1=\hat{\openone} - \hat{M}_0$, we can rewrite the left-hand side of Eq.~(\ref{eq:helstrom_sum_bound}) as
\begin{eqnarray}
\tr{\hat{M}_0 \hat{\rho}_0^{\otimes d}}+\tr{(\hat{\openone}-\hat{M}_0) \hat{\rho}_1^{\otimes d}} = \tr{\hat{\rho}_1^{\otimes d}} + \tr{\hat{M}_0 (\hat{\rho}_0^{\otimes d} - \hat{\rho}_1^{\otimes d})} = 1 + \tr{\hat{M}_0\,\hat{\Delta}}.
\label{eq:rewrite_sum}
\end{eqnarray}
Since $\{\hat{M}_0, \hat{M}_1\}$ is a POVM, $0 \preceq \hat{M}_0 \preceq \hat{\openone}$. Let $\hat{\Delta}=\hat{\Delta}_+-\hat{\Delta}_-$ be the Jordan decomposition, where $\hat{\Delta}_+,\hat{\Delta}_-\succeq 0$ have orthogonal supports. Then,
\begin{eqnarray}
\tr{\hat{M}_0\,\hat{\Delta}} = \tr{\hat{M}_0 \hat{\Delta}_+} - \tr{\hat{M}_0 \hat{\Delta}_-} \le \tr{\hat{\Delta}_+},
\label{eq:bound_by_positive_part}
\end{eqnarray}
because $\tr{\hat{M}_0 \hat{\Delta}_+} \le \tr{\hat{\Delta}_+}$ (as $\hat{M}_0 \preceq \hat{\openone}$) and $\tr{\hat{M}_0 \hat{\Delta}_-} \ge 0$. Moreover, since $\tr{\hat{\Delta}}=0$, we have $\tr{\hat{\Delta}_+} = \tr{\hat{\Delta}_-} = \frac{1}{2}\|\hat{\Delta}\|_1$. By substituting into Eq.~(\ref{eq:rewrite_sum}), we can yield Eq.~(\ref{eq:helstrom_sum_bound}).

For achievability, take $\hat{M}_0$ to be the projector onto $\textrm{supp}(\hat{\Delta}_+)$, so that $\tr{\hat{M}_0 \hat{\Delta}_+} = \tr{\hat{\Delta}_+}$ and $\tr{\hat{M}_0 \hat{\Delta}_-}=0$, hence $\tr{\hat{M}_0 \hat{\Delta}} = \tr{\hat{\Delta}_+} = \frac{1}{2}\|\hat{\Delta}\|_1$, which saturates Eq.~(\ref{eq:bound_by_positive_part}).
\end{proof}

{\bf Theorem~\ref{thm:helstrom_worst_case}} immediately yields a PL impossibility statement: if one demands near-certain identification of both environments, then the $d$-copy trace distance must be correspondingly close to its maximal value. In particular, for non-orthogonal pure states this enforces a quantitative lower bound on the number of available copies.

\begin{lemma}[Trace distance of pure states]
\label{lem:pure_state_trace_distance}
Let $\ket{\psi}, \ket{\phi}$ be unit vectors and let $\hat{\rho}=\ketbra{\psi}{\psi}$, $\hat{\sigma}=\ketbra{\phi}{\phi}$. Then,
\begin{eqnarray}
\|\hat{\rho}-\hat{\sigma}\|_1 = 2\sqrt{1- \abs{\braket{\psi}{\phi}}^2}.
\label{eq:pure_state_trace_distance}
\end{eqnarray}
Consequently,
\begin{eqnarray}
\left\|\hat{\rho}^{\otimes d} - \hat{\sigma}^{\otimes d}\right\|_1 = 2\sqrt{1-\abs{\braket{\psi}{\phi}}^{2d}}.
\label{eq:pure_state_trace_distance_tensor}
\end{eqnarray}
\end{lemma}

\begin{proof}
Let $c := \braket{\psi}{\phi}$ and assume $\abs{c} \in (0,1)$ (the endpoints are trivial). Define the normalized vector
\begin{eqnarray}
\ket{\phi_\perp} := \frac{\ket{\phi} - c\ket{\psi}}{\sqrt{1 - \abs{c}^2}},
\label{eq:phi_perp_def}
\end{eqnarray}
which satisfies $\braket{\psi}{\phi_\perp}=0$ and $\|\phi_\perp\|=1$. In the orthonormal basis, $\{\ket{\psi}, \ket{\phi_\perp}\}$, the operators $\hat{\rho}$ and $\hat{\sigma}$ have matrix representations
\begin{eqnarray}
\hat{\rho} \equiv
\left(
\begin{array}{cc}
1 & 0\\
0 & 0
\end{array}
\right),
\quad
\hat{\sigma} \equiv
\left(
\begin{array}{cc}
\abs{c}^2 & c\sqrt{1-\abs{c}^2}\\
c^\ast\sqrt{1-\abs{c}^2} & 1-\abs{c}^2
\end{array}
\right).
\label{eq:rho_sigma_matrix}
\end{eqnarray}
Therefore, $\hat{\rho}-\hat{\sigma}$ is represented by
\begin{eqnarray}
\hat{\rho}-\hat{\sigma} \equiv
\left(
\begin{array}{cc}
1-\abs{c}^2 & -c\sqrt{1-\abs{c}^2}\\
-c^\ast\sqrt{1-\abs{c}^2} & -(1-\abs{c}^2)
\end{array}
\right).
\label{eq:diff_matrix}
\end{eqnarray}
A direct computation shows that this $2 \times 2$ Hermitian matrix has eigenvalues $\pm\sqrt{1-\abs{c}^2}$. Hence its trace norm, which is the sum of absolute eigenvalues, equals $2\sqrt{1-\abs{c}^2}$, proving Eq.~(\ref{eq:pure_state_trace_distance}). For Eq.~(\ref{eq:pure_state_trace_distance_tensor}), by applying Eq.~(\ref{eq:pure_state_trace_distance}) with $\braket{\psi^{\otimes d}}{\phi^{\otimes d}}=c^d$, we can yield the claim.
\end{proof}

\begin{corollary}[Copy complexity for reliable identification]
\label{cor:copy_complexity_bound}
Let $\ket{\psi}, \ket{\phi}$ be distinct non-orthogonal pure states with overlap $\gamma:=\abs{\braket{\psi}{\phi}} \in (0,1)$. Consider the binary task in Eq.~(\ref{eq:binary_utility}) with environments $\hat{\rho}_0=\ketbra{\psi}{\psi}$ and $\hat{\rho}_1=\ketbra{\phi}{\phi}$ under $d$-copy quantum access. If a protocol achieves
\begin{eqnarray}
\tr{\hat{M}_0 \hat{\rho}_0^{\otimes d}} \ge 1-\delta
\quad\mbox{and}\quad
\tr{\hat{M}_1 \hat{\rho}_1^{\otimes d}} \ge 1-\delta,
\label{eq:two_sided_success}
\end{eqnarray}
then necessarily
\begin{eqnarray}
\delta \ge \frac{1-\sqrt{1-\gamma^{2d}}}{2}.
\label{eq:delta_lower_bound}
\end{eqnarray}
Equivalently, to make $\delta<1/2$, one must have
\begin{eqnarray}
d \ge \frac{1}{-2\ln \gamma}\ln\left(\frac{1}{4\delta(1-\delta)}\right).
\label{eq:d_lower_bound}
\end{eqnarray}
\end{corollary}

\begin{proof}
Summing Eq.~(\ref{eq:two_sided_success}) gives $\tr{\hat{M}_0 \hat{\rho}_0^{\otimes d}} + \tr{\hat{M}_1\,\hat{\rho}_1^{\otimes d}} \ge 2(1-\delta)$. By {\bf Theorem~\ref{thm:helstrom_worst_case}} and {\bf Lemma~\ref{lem:pure_state_trace_distance}},
\begin{eqnarray}
2(1-\delta) \le 1 + \frac{1}{2}\left\|\hat{\rho}_0^{\otimes d} - \hat{\rho}_1^{\otimes d}\right\|_1 = 1+\sqrt{1-\gamma^{2d}}.
\label{eq:combine_bounds}
\end{eqnarray}
Rearranging yields Eq.~(\ref{eq:delta_lower_bound}). For~(\ref{eq:d_lower_bound}), note that Eq.~(\ref{eq:delta_lower_bound}) implies $\sqrt{1-\gamma^{2d}}\ge 1-2\delta$, hence $\gamma^{2d} \le 1-(1-2\delta)^2=4\delta(1-\delta)$, and taking logarithms gives Eq.~(\ref{eq:d_lower_bound}).
\end{proof}

{\bf Corollary~\ref{cor:copy_complexity_bound}} illustrates a distinct, genuinely physical source of limitation that is invisible in the set-theoretic view of learnability. The obstacle is not logical independence but the geometry of quantum state space: two non-orthogonal states remain partially confusable no matter what measurement is applied to finitely many copies. In classical learning theory, one might attempt to circumvent such confusability by duplicating the sample and repeating a test; in quantum PL, no-cloning precludes this maneuver. The number of copies is therefore not merely a book-keeping parameter but a fundamental resource, and PL makes the dependence of learnability on this resource mathematically explicit.

\subsection{No-signaling constraints and decidability in finite operational models}

The previous case studies addressed how physical constraints change the truth of learnability statements by changing the object quantified over. There is another, complementary sense in which operational constraints can enhance the foundational status of learnability: they can render the PL feasibility question \emph{algorithmically decidable} from finite data. The key is that many physically motivated admissible classes are describable by finitely many linear constraints (no-signaling) or by convex semidefinite constraints (quantum)~\cite{BangChoJae2026}, so that optimizing or verifying performance reduces to standard convex feasibility.

We formalize this in a minimal finite setting. Let $\Theta$ and $\mathcal{H}$ be finite sets. For each sample budget $d$, identify a Markov kernel $Q \in \hat{\Delta}(\mathcal{H})^\Theta$ with the vector of probabilities $q_{\theta,h}:=Q(h | \theta)$. Suppose the admissible set $\mathfrak{L}_d$ is specified as a rational polytope in this finite-dimensional space.

\begin{definition}[Rational polytope admissibility]
\label{def:rational_polytope}
An admissible set $\mathfrak{L}_d \subseteq \hat{\Delta}(\mathcal{H})^\Theta$ is a rational polytope if there exist matrices $A_d$ and $B_d$ with rational entries such that $Q \in \mathfrak{L}_d$ if and only if its coordinate vector $q$ satisfies
\begin{eqnarray}
q A_d \le B_d
\label{eq:polytope_constraints}
\end{eqnarray}
(including the simplex constraints $q_{\theta,h} \ge 0$ and $\sum_h q_{\theta,h}=1$).
\end{definition}

\begin{theorem}[Decidability of PL feasibility under polytope constraints~\cite{BoydVandenberghe2004}]
\label{thm:finite_model_decidability}
Let $(\Theta,\mathcal{H},U)$ be a learning task with finite $\Theta$ and $\mathcal{H}$, and fix $\epsilon,\delta \in (0,1)$. For a fixed sample budget $d$, assume $\mathfrak{L}_d$ is a rational polytope in the sense of {\bf Definition~\ref{def:rational_polytope}}. Then, deciding whether there exists $Q\in\mathfrak{L}_d$ satisfying the PL guarantee in Eq.~(\ref{eq:pl_guarantee}) reduces to a linear feasibility problem over rational constraints, and is therefore algorithmically decidable.
\end{theorem}

\begin{proof}
For each $\theta \in \Theta$, let
\begin{eqnarray}
\mathcal{G}_\theta(\epsilon)\;:=\;\left\{h\in\mathcal{H}:\ U(\theta,h)\ \ge\ \textrm{opt}_{\mathcal{H}}(\theta)-\epsilon\right\}.
\label{eq:good_set}
\end{eqnarray}
Because $\mathcal{H}$ is finite, $\textrm{opt}_{\mathcal{H}}(\theta)$ is attained and $\mathcal{G}_\theta(\epsilon)$ is well-defined. The PL condition in Eq.~(\ref{eq:pl_guarantee}) is equivalent to the family of linear inequalities
\begin{eqnarray}
\sum_{h\in\mathcal{G}_\theta(\epsilon)} q_{\theta,h} \ge 1-\delta, \quad (\forall \, \theta \in \Theta),
\label{eq:pl_linear_constraints}
\end{eqnarray}
together with $Q \in \mathfrak{L}_d$. By {\bf Definition~\ref{def:rational_polytope}}, the latter is equivalent to the finite set of rational linear inequalities as in Eq.~(\ref{eq:polytope_constraints}). Therefore, the existence of $Q$ satisfying PL is exactly the feasibility of a rational linear system, which is decidable by standard linear programming methods.
\end{proof}

No-signaling constraints provide a canonical and physically meaningful source of polytope admissibility: for finite input--output alphabets, the no-signaling conditions of {\bf Definition~\ref{def:no_signaling}} are linear equalities, and together with nonnegativity and normalization they define a convex polytope. Consequently, whenever a learning scenario can be cast so that admissible learner behaviors are parameterized by a finite no-signaling correlation table, the associated PL feasibility question is decidable in the concrete sense of {\bf Theorem~\ref{thm:finite_model_decidability}}. This does not ``decide'' the set-theoretic independence exhibited by EMX over the continuum. Instead, it illustrates a different and operationally relevant principle: once one commits to a physically specified constraint set of finite description, the question ``does there exist an admissible procedure achieving a target performance?'' becomes a computational question about a finite convex set, rather than a logical question about the existence of arbitrary functions on infinite domains.

The three case studies thus converge on a common methodological message. PL does not attempt to repair the set-theoretic foundations of unconstrained learnability by importing physical axioms as new mathematical postulates. Rather, it changes the object of study from ``functions that happen to satisfy a statistical inequality'' to ``operational protocols permitted by a specified physical model.'' In doing so, it both removes sources of purely set-theoretic indeterminacy that arise from non-operational idealizations and reveals new, intrinsically physical barriers---such as copy complexity under no-cloning---that shape what it means to learn in the real world.

\section{Discussions and outlook}

\subsection{What does it mean to ``decide learnability'' in the physical world?}

The word ``decide'' carries an ambiguity that is harmless in much of classical learning theory but becomes unavoidable once one confronts independence phenomena. In the PAC/VC paradigm, the statement ``$\mathcal{H}$ is learnable'' is simultaneously (i) a well-defined mathematical proposition and (ii) an operational claim that a concrete procedure exists. These two readings coincide because the existential quantifier ranges over objects whose existence can be certified by finitary constructions and whose analysis reduces to finite witnesses. The EMX example of Ben-David \emph{et al.} shows that this coincidence is not a theorem of nature but a special feature of the classification setting~\cite{BenDavid2019}. In that example, the mathematical proposition ``there exists a learner'' is sensitive to set-theoretic axioms, while the operational content of learning---which is always mediated by finite records and finite interventions---is left implicit.

The physics-aware learnability (PL) resolves this ambiguity by making the interface between the learner and the world explicit. The definition of PL learnability ({\bf Definition~\ref{def:pl_learnability}}) is not a new performance criterion; it is a different location of the existential quantifier. The question is no longer ``does there exist an arbitrary set-theoretic function with a certain property?'' but rather ``does there exist an admissible physical protocol, within a specified access model, that achieves the same property?'' The difference is subtle in notation and profound in meaning.

This shift has two consequences that, taken together, clarify what it means to ``decide learnability'' in the physical world.

\smallskip\noindent
{\bf First, PL disentangles set-theoretic indeterminacy from operational feasibility.}
The Ben-David construction is not ``refuted'' by PL. Rather, PL asks a different question: what aspects of the continuum are actually interrogated by a physically admissible sampling interface? Under coarse-graining, the learner sees only a countable alphabet and can output only hypotheses that are finitely nameable at that resolution. In that setting, the relevant EMX instance becomes provably learnable ({\bf Corollary~\ref{cor:finite_precision_emx}}), and one can even write down explicit learners with transparent sample complexity ({\bf Corollary~\ref{cor:quantile_emx_coarse_grain}}). The moral is not that physics settles cardinal arithmetic, but that independence results can be artifacts of allowing hypotheses and learners to depend on mathematical structure that no physical interface can supply.

\smallskip\noindent
{\bf Second, PL turns ``deciding learnability'' into a feasibility question about a finitely specified constraint set.}
Once an access model is fixed, the admissible behaviors of a learner constitute a set $\mathfrak{L}_d$ of conditional output laws. If the environment class and the hypothesis class are finitely described, and $\mathfrak{L}_d$ is itself given by finitely many physical constraints, then the PL question becomes a finite mathematical problem: does a certain convex set intersect a certain performance region? In this regime, ``deciding learnability'' is closer in spirit to checking the feasibility of a design specification than to settling a set-theoretic statement.

The following theorem formalizes this operational reading in two representative physical models. The first is the no-signaling model (and, more broadly, any model whose admissible behaviors form a rational polytope); the second is the finite-dimensional quantum model (where admissible behaviors arise from POVMs on $d$ copies, hence form a semidefinite-representable set).

\begin{theorem}[Operational decidability in finite PL models]
\label{thm:operational_decidability}
Let $(\Theta,\mathcal{H},U)$ be a learning task with finite $\Theta$ and finite $\mathcal{H}$. Fix $\epsilon,\delta \in (0,1)$ and a sample budget $d \in \mathbb{N}$. Define for each $\theta \in \Theta$ the set of $\epsilon$-optimal hypotheses
\begin{eqnarray}
\mathcal{G}_\theta(\epsilon) := \left\{h \in \mathcal{H}: U(\theta,h) \ge \mathrm{opt}_{\mathcal{H}}(\theta)-\epsilon\right\}.
\label{eq:good_set_discussion}
\end{eqnarray}
Consider the PL feasibility question: does there exist an admissible kernel $Q\in\mathfrak{L}_d$, such that $\sum_{h \in \mathcal{G}_\theta(\epsilon)} Q(h|\theta) \ge 1-\delta \quad (\forall \, \theta \in \Theta)$?
Then the following hold.
\begin{itemize}
\item[(i)] If $\mathfrak{L}_d$ is a rational polytope ({\bf Definition~\ref{def:rational_polytope}}), the feasibility question reduces to a linear feasibility problem over rational constraints, and is algorithmically decidable.
\item[(ii)] Suppose $\Theta=\{1,\ldots,\abs{\Theta}\}$ indexes a finite set of density operators $\{\hat{\rho}_\theta\}_{\theta \in \Theta}$ on a finite-dimensional Hilbert space $\mathcal{K}$, and suppose $\mathfrak{L}_d$ is the quantum PL admissible set consisting of kernels realizable by POVMs on $d$ copies, i.e. kernels of the form
\begin{eqnarray}
Q(h | \theta) = \tr{\hat{M}_h \hat{\rho}_\theta^{\otimes d}} \quad\mbox{with}\quad \hat{M}_h \succeq 0, \ \ \sum_{h\in\mathcal{H}} \hat{M}_h=\hat{\openone}_{\mathcal{K}^{\otimes d}}.
\label{eq:quantum_kernel_form_discussion}
\end{eqnarray}
\end{itemize}
Then, the feasibility question reduces to semidefinite feasibility with linear matrix inequalities. In particular, feasibility can be decided to arbitrary numerical precision by standard SDP methods.
\end{theorem}

\begin{proof}
Part (i) is exactly the content of {\bf Theorem~\ref{thm:finite_model_decidability}}. We include it here to emphasize the interpretational point that polyhedral physical constraints turn PL feasibility into a finite linear system.

For part (ii), introduce matrix variables $\{\hat{M}_h\}_{h \in \mathcal{H}}$ on $\mathcal{K}^{\otimes d}$. The condition that these variables define a POVM is the semidefinite constraint
\begin{eqnarray}
\hat{M}_h \succeq 0 \quad (\forall \, h \in \mathcal{H}), \quad \sum_{h \in \mathcal{H}} \hat{M}_h = \hat{\openone}_{\mathcal{K}^{\otimes d}}.
\label{eq:povm_sdp_constraints}
\end{eqnarray}
For each $\theta\in\Theta$, the PL constraint becomes, using~(\ref{eq:quantum_kernel_form_discussion}),
\begin{eqnarray}
\sum_{h \in \mathcal{G}_\theta(\epsilon)} \tr{\hat{M}_h \hat{\rho}_\theta^{\otimes d}} \ge 1-\delta.
\label{eq:quantum_pl_constraints}
\end{eqnarray}
Since $\tr{\hat{M}_h\,\hat{\rho}_\theta^{\otimes d}}$ is linear in $\hat{M}_h$, the inequalities Eq.~(\ref{eq:quantum_pl_constraints}) are linear constraints on the matrix variables. Therefore, the existence of a quantum-admissible $Q$ satisfying PL is equivalent to feasibility of the semidefinite system in Eq.~(\ref{eq:povm_sdp_constraints}) and Eq.~(\ref{eq:quantum_pl_constraints}). The numerical decidability to arbitrary precision follows from standard SDP feasibility algorithms.
\end{proof}

{\bf Theorem~\ref{thm:operational_decidability}} suggests a concrete answer to the motivating question ``what does it mean to decide learnability?'' In the physical world, the question is not whether a statement is provable in ZFC, but whether a device meeting a specification exists within a physically admissible constraint set. When this constraint set is finitely described---as it is for finite-alphabet no-signaling models or finite-dimensional quantum models with finitely many relevant environments---``deciding learnability'' becomes an explicit feasibility problem, and the relevant mathematics is convex optimization rather than set-theoretic independence.

At the same time, PL does not promise that all foundational difficulties disappear. Physics can also introduce new kinds of obstruction. No-cloning makes sample size a hard resource ({\bf Corollary~\ref{cor:copy_complexity_bound}}), and no-signaling can restrict the admissible flow of information in distributed protocols. Moreover, as one enlarges the admissible model class (for example, by allowing unbounded-dimensional quantum systems or by permitting complex adaptive interactions), feasibility questions that are convex in small finite models can become dramatically more subtle. In this sense, PL replaces a single notion of ``decidability'' by a more faithful stratification: some questions are mathematically independent when the learner is an arbitrary function, some become operationally decidable under finite physical descriptions, and some may remain computationally intractable or even undecidable when the admissible physics is sufficiently expressive. The essential gain is that PL makes explicit which sense of ``decide'' is being invoked and where the difficulty truly resides.

\subsection{Open problems \& outlook}

The PL framework is intentionally a beginning rather than an endpoint. Its main purpose is to provide a clean interface for asking questions that are simultaneously learning-theoretic and physically meaningful. We highlight several directions where the definitions introduced here suggest precise problems whose solutions would deepen the connection between learning theory and the physics of information.

\subsubsection*{PL compression and effective structure}
The Ben-David equivalence between weak EMX learnability and monotone compression relies on union-boundedness and on quantifying over set-theoretic learners. In PL, the learner is a constrained protocol and the output hypotheses are representable objects. A natural question is whether there exists an operational analogue of monotone compression, formulated in terms of admissible transformations on data registers, that is equivalent to PL learnability for broad classes of tasks. One expects that such a notion, if it exists, would refine classical compression schemes by incorporating the resource that the physics actually constrains: copies, communication, or measurement access. Establishing a PL compression theorem would provide a structural backbone for the framework analogous to the role played by VC dimension in PAC learning and by monotone compression in EMX.

\subsubsection*{Dimension-like characterizations under restricted physics}
Ben-David \emph{et al.} show that no dimension of ``finite character'' can characterize EMX learnability in full generality unless ZFC is inconsistent~\cite{BenDavid2019}. PL suggests a sharper conjecture: while no absolute dimension parameter exists for general learning tasks, there may exist relative invariants, depending on $\mathfrak{L}$, that characterize PL learnability within operationally meaningful regimes. For example, one may ask whether coarse-grained EMX classes admit a dimension-like characterization on the induced discrete alphabet, or whether certain quantum PL tasks admit invariants tied to state distinguishability, channel capacities, or information radius. The challenge is to define such invariants in a way that is both physically interpretable and mathematically robust.

\subsubsection*{Coarse-graining as a design variable}
In PL, finite precision is not merely a nuisance but a structural feature of the access model. This suggests a design question: given a task on a continuum domain, what is the coarsest interface $\pi$ (or the cheapest measurement in a quantum setting) that renders the task learnable to a desired accuracy? Conversely, how does the achievable accuracy scale as a function of resolution, measurement noise, or sampling budget? These questions recast learning as the joint design of a hypothesis-selection rule and a measurement interface, and they naturally lead to trade-offs that combine statistical complexity with physical resource costs.

\subsubsection*{Gaps between quantum and no-signaling learnability}
Quantum mechanics is no-signaling ({\bf Theorem~\ref{thm:quantum_no_signaling}}), but the converse is not true: no-signaling correlations form a strictly larger set than quantum correlations in many settings~\cite{Brunner2014Bell}. This separation suggests the possibility of learning tasks that are feasible under no-signaling admissibility but infeasible under quantum admissibility, or vice versa when additional structure is imposed. Formulating explicit tasks in which such gaps can be proved, and relating those gaps to information principles (such as information causality or uncertainty relations), would clarify whether ``physics-aware'' should be understood as ``quantum-aware'' or whether more general operational principles are the right primitive.

\subsubsection*{The decidability frontier for PL}
{\bf Theorem~\ref{thm:operational_decidability}} identifies a regime in which PL feasibility is decidable because the admissible behaviors form a finitely described convex set. A natural program is to map the boundary of this regime. Which extensions of $\mathfrak{L}$ preserve convexity and finite describability? When does the PL decision problem become computationally hard even though it remains decidable? Under what conditions can it become undecidable in the algorithmic sense? This problem is not merely about complexity theory; it is a foundational question about which physical descriptions yield a mathematically tractable theory of learnability and which do not~\cite{BangChoJae2026}.

\subsection{Conclusion}

The ambition of the conventional learning theory is to ``identify the learnable.'' Ben-David \emph{et al.} remind us that this ambition is inseparable from the formal language in which we ask it. In other words, when the learners are treated as arbitrary functions over infinite domains, even simple-looking learning questions can become independent of the standard axioms of mathematics~\cite{BenDavid2019}. PL reframes this lesson as an opportunity. By treating learners as admissible physical protocols and by making the data interface explicit, PL shifts learnability from a statement about set-theoretic existence to a statement about operational feasibility.

This shift has two complementary effects. It can remove purely set-theoretic indeterminacy that arises from non-operational idealizations, as in the finite-precision reinterpretation of the EMX example. At the same time, it can reveal obstacles that are intrinsically physical, such as the copy complexity induced by quantum distinguishability and the restrictions imposed by no-signaling. In this sense, PL does not ``solve'' undecidability by importing physics as additional axioms. Instead, it proposes that the meaningful notion of learnability is inherently relative: relative to a measurement interface, relative to admissible transformations, and relative to the resources that physics makes scarce.

The broader outlook is that learning theory and the foundations of physics share a common theme: both fields study what can be inferred about an unknown world from finite interaction with it. Where classical learning theory often assumes an idealized channel from world to data, PL asks us to write that channel down and take its laws seriously. Doing so does not diminish the mathematical nature of learnability; it sharpens it, by ensuring that the theorems we prove are about learners that could, in principle, exist.

\appendix

\section{Proof of the no-cloning theorem (Theorem~\ref{thm:no_cloning})}\label{appendix:nocloning}

Assume, towards a contradiction, that there exists a linear isometry $\hat{U}:\mathcal{H}\otimes\mathcal{H} \to \mathcal{H}\otimes\mathcal{H}$ and a fixed unit ``blank'' state $\ket{0} \in \mathcal{H}$ such that
\begin{eqnarray}
\hat{U}\big(\ket{\psi}\otimes\ket{0}\big)=\ket{\psi}\otimes\ket{\psi}
\end{eqnarray}
for every unit vector $\ket{\psi} \in \mathcal{H}$.

Because $\dim(\mathcal{H}) \ge 2$, we can choose two distinct non-orthogonal unit vectors. For example, fix an orthonormal pair $\ket{e_0}, \ket{e_1}$ and take $\ket{\psi} = \ket{e_0}$ and $\ket{\phi}=\cos\theta \ket{e_0} + \sin\theta \ket{e_1}$ for some $\theta \in (0,\pi/2)$. Then $0 < \abs{\braket{\psi}{\phi}}=\cos\theta < 1$.

Now compute the inner product of the two purportedly cloned states in two ways. First, by definition of $\hat{U}$,
\begin{eqnarray}
\hat{U}\ket{\psi}\otimes\ket{0} = \ket{\psi}\otimes\ket{\psi},
\quad
\hat{U}\ket{\phi}\otimes\ket{0} = \ket{\phi}\otimes\ket{\phi},
\end{eqnarray}
Second, since $\hat{U}$ is an isometry, it preserves inner products:
\begin{eqnarray}
\bigl( \bra{\psi}\otimes\bra{0} \bigr) \bigl( \ket{\phi}\otimes\ket{0} \bigr) = \left( \bra{\psi}\otimes\bra{0} \hat{U}^\dagger \right) \left( \hat{U}\ket{\phi}\otimes\ket{0} \right) 
\end{eqnarray}
The left-hand side is $\braket{\psi}{\phi}\braket{0}{0}=\braket{\psi}{\phi}$, while the right-hand side equals
\begin{eqnarray}
\bigl( \bra{\psi}\otimes\bra{\psi} \bigr) \bigl( \ket{\phi}\otimes\ket{\phi} \bigr) = \braket{\psi}{\phi}\braket{\psi}{\phi} = \braket{\psi}{\phi}^2.
\end{eqnarray}
Therefore, $\braket{\psi}{\phi}=\braket{\psi}{\phi}^2$. Let $c:=\braket{\psi}{\phi}$. The equality $c=c^2$ implies $c \in \{0,1\}$ (over $\mathbb{C}$ as well), contradicting $0 < \abs{c} < 1$. This contradiction completes the proof.

\section{Proof of quantum no-signaling (Theorem~\ref{thm:quantum_no_signaling})}\label{appendix:nosignaling}

Fix any bipartite state $\hat{\rho}_{AB}$ on $\mathcal{H}_A\otimes\mathcal{H}_B$, any POVMs $\{\hat{M}_a^x\}_{a \in A}$ on $\mathcal{H}_A$ indexed by $x \in X$, and any POVMs $\{\hat{N}_b^y\}_{b \in B}$ on $\mathcal{H}_B$ indexed by $y \in Y$. Define
\begin{eqnarray}
p(a,b | x,y) := \tr{(\hat{M}_a^x \otimes \hat{N}_b^y) \hat{\rho}_{AB}}.
\end{eqnarray}

We verify the two no-signaling conditions. Fix $y \in Y$ and $b \in B$. Using linearity of trace and POVM completeness on $A$, i.e. $\sum_{a \in A}\hat{M}_a^x=\hat{\openone}_A$ for every $x$, we compute
\begin{eqnarray}
\sum_{a \in A} p(a,b | x,y) = \sum_{a \in A}\tr{(\hat{M}_a^x\otimes \hat{N}_b^y)\hat{\rho}_{AB}} = \tr{\Big(\sum_{a \in A}\hat{M}_a^x\Big) \otimes \hat{N}_b^y \hat{\rho}_{AB}} = \tr{(\hat{\openone}_A \otimes \hat{N}_b^y) \hat{\rho}_{AB}},
\end{eqnarray}
which is independent of $x$. This proves Eq.~(\ref{eq:no_signal_A_to_B}).

The second condition is proved analogously. Fix $x \in X$ and $a \in A$, and use completeness on $B$, $\sum_{b \in B}\hat{N}_b^y=\hat{\openone}_B$ for every $y$, to obtain
\begin{eqnarray}
\sum_{b\in B} p(a,b | x,y) = \tr{(\hat{M}_a^x \otimes \hat{\openone}_B) \hat{\rho}_{AB}},
\end{eqnarray}
which is independent of $y$. This proves Eq.~(\ref{eq:no_signal_B_to_A}) and completes the proof.


%

\end{document}